\documentclass[11pt]{article} 

\usepackage[letterpaper,margin=1in]{geometry}

\usepackage[letterpaper,margin=1in]{geometry}

\newif\ifred
\redtrue 

\usepackage[T1]{fontenc}
\usepackage{microtype}

\usepackage{titlesec}
\titleformat*{\paragraph}{\bfseries}

\usepackage{amsthm,amsmath,amssymb,amsfonts,amssymb,mathtools}
\usepackage{amsmath,amssymb,amsfonts,amssymb,mathtools}
\usepackage{xcolor}
\usepackage{empheq}
\usepackage{dsfont}
\usepackage{mathrsfs}
\usepackage{graphicx}
\usepackage{enumitem}
\usepackage{aliascnt} 
\usepackage{xspace}
\usepackage{bbm}

\usepackage{caption}
\usepackage{tikz}
\usetikzlibrary{patterns}
\usetikzlibrary{patterns}
\usetikzlibrary{arrows,shapes,automata,backgrounds,petri,positioning}
\usetikzlibrary{shadows}
\usetikzlibrary{calc}
\usetikzlibrary{spy}
\usetikzlibrary{angles, quotes}
\usetikzlibrary{matrix}
\usepackage{pgf,pgfplots}
\usepackage{pgfmath,pgffor}
\pgfplotsset{compat=1.17}

\usepackage{framed}
 \usepackage{algorithm}
\usepackage[noend]{algpseudocode}

\definecolor[named]{ACMBlue}{cmyk}{1,0.1,0,0.1}
\definecolor[named]{ACMYellow}{cmyk}{0,0.16,1,0}
\definecolor[named]{ACMOrange}{cmyk}{0,0.42,1,0.01}
\definecolor[named]{ACMRed}{cmyk}{0,0.90,0.86,0}
\definecolor[named]{ACMLightBlue}{cmyk}{0.49,0.01,0,0}
\definecolor[named]{ACMGreen}{cmyk}{0.20,0,1,0.19}
\definecolor[named]{ACMPurple}{cmyk}{0.55,1,0,0.15}
\definecolor[named]{ACMDarkBlue}{cmyk}{1,0.58,0,0.21}

\usepackage[colorlinks,citecolor=blue,linkcolor=magenta,bookmarks=true]{hyperref}
 \usepackage[nameinlink]{cleveref}
\creflabelformat{ineq}{#2{\upshape(#1)}#3}
\crefname{sub}{Subsection}{Subsection}
\creflabelformat{Subsection}{#2{\upshape(#1)}#3}
\crefname{sdp}{SDP}{SDP}
\creflabelformat{sdp}{#2{\upshape(#1)}#3}
\crefname{lp}{LP}{LP}
\creflabelformat{lp}{#2{\upshape(#1)}#3}
\usepackage{aliascnt}
\usepackage{cleveref}
\crefname{ineq}{Inequality}{Inequality}
\creflabelformat{ineq}{#2{\upshape(#1)}#3}
\crefname{sub}{Subsection}{Subsection}
\creflabelformat{Subsection}{#2{\upshape(#1)}#3}
\crefname{sdp}{SDP}{SDP}
\creflabelformat{sdp}{#2{\upshape(#1)}#3}
\crefname{lp}{LP}{LP}
\creflabelformat{lp}{#2{\upshape(#1)}#3}



\newtheorem{theorem}{Theorem}[section]

\newtheorem{lemma}{Lemma}
\newtheorem{informal theorem}[theorem]{Theorem (informal statement)}

\newtheorem{corollary}[theorem]{Corollary}
 \newtheorem{claim}[theorem]{Claim}

\newtheorem{definition}[theorem]{Definition}

\newcommand{\lp}{\left}
\newcommand{\rp}{\right}
\newcommand\norm[1]{\left\| #1 \right\|}

\DeclareMathOperator*{\pr}{\mathbf{Pr}}
\DeclareMathOperator*{\E}{\mathbf{E}}
\newcommand{\proj}{\mathrm{proj}}

\renewcommand{\H}{\mathcal H }
\newcommand{\X}{\mathcal X }

\newcommand{\R}{\mathbb{R}}

\newcommand{\poly}{\mathrm{poly}}
\newcommand{\dom}{\mathrm{dom}}

\newcommand{\sgn}{\mathrm{sign}}

\newcommand{\abs}[1]{\lp| #1 \rp|}
\newcommand{\A}{\mathcal{A}}

\newcommand{\card}[1]{|#1|}

\renewcommand\Pr{\pr}

\begin{document}

\title{Active Learning with Simple Questions}

\author{
Vasilis Kontonis \\
UT Austin\\
\texttt{vasilis@cs.utexas.edu}
\and
Mingchen Ma\\
UW-Madison\\
\texttt{mingchen@cs.wisc.edu}
\and
Christos Tzamos\\
University of Athens and Archimedes AI\\
\texttt{tzamos@wisc.edu}
}

\maketitle

\begin{abstract}%

 We consider an active learning setting where a learner is presented with a pool $S$ of $n$ unlabeled examples belonging to a domain $\mathcal X$ and asks queries to find the underlying labeling that agrees with a target concept $h^\ast \in \mathcal H$. 
   
    In contrast to traditional active learning that queries a single example for its label, we study more general \emph{region queries} that allow the learner to pick a subset of the domain $T \subset \mathcal X$ and a target label $y$ and ask a labeler whether $h^\ast(x) = y $ for every example in the set $T \cap S$. Such more powerful queries allow us to bypass the limitations of traditional active learning and use significantly fewer rounds of interactions to learn but can potentially lead to a significantly more complex query language. Our main contribution is quantifying the trade-off between the number of queries and the complexity of the query language used by the learner.

    We measure the complexity of the region queries via the VC dimension of the family of regions. We show that given any hypothesis class $\H$ with VC dimension $d$, one can design a region query family $Q$ with VC dimension $O(d)$ such that for every set of $n$ examples $S \subset \X$ and every $h^* \in \H$, a learner can submit $O(d\log n)$ queries from $Q$ to a labeler and perfectly label $S$. We show a matching lower bound by designing a hypothesis class $\H$ with VC dimension $d$ and a dataset $S \subset \X$ of size $n$ such that any learning algorithm using any query class with VC dimension less than $O(d)$ must make $\poly(n)$ queries to label $S$ perfectly.

    Finally, we focus on well-studied hypothesis classes including unions of intervals, high-dimensional boxes, and $d$-dimensional halfspaces, and obtain stronger results. In particular, we design learning algorithms that (i) are computationally efficient and (ii) work even when the queries are not answered based on the learner's pool of examples $S$ but on some unknown superset $L$ of $S$.

\end{abstract}


\section{Introduction}

Acquiring labeled examples is often challenging in applications as querying either human annotators
or powerful pre-trained models is time consuming and/or expensive. Active learning aims to minimize
the number of labeled examples required for a task by allowing the learner to adaptively select
for which examples they want to obtain labels. More precisely, in pool-based active learning, 
the learner has to infer all labels of a pool $S$ of $n$ unlabeled examples, and can adaptively select an example $x \in S$ and ask for its label. 

Even though it is known that active learning can exponentially reduce the number of required 
labels, this is unfortunately only true in very idealized settings such as datasets labeled by one-dimensional thresholds or structured high-dimensional instances (e.g., Gaussian marginals) \cite{dasgupta2005analysis, balcan2007margin,balcan2013active,balcan2017sample,awasthi2017power}.
It is well-known that without such distributional assumptions, even in $2$ dimensions, 
linear classification active learning yields no improvement over passive learning \cite{dasgupta2004analysis,dasgupta2005coarse}. 

\paragraph{Active Learning with Queries}
To bypass the hardness results and establish learning without restrictive distributional 
assumptions
\cite{balcan2012robust,kane2017active,hopkins2020noise,hopkins2021bounded,yona2022active,bressan2022active} introduce enriched queries, where the learner is allowed to make more complicated queries. 
In this work we follow this paradigm and aim to characterize the trade-off between the number of 
required queries and their complexity. For example, comparison queries that select two examples
and ask which one is closer to the decision boundary \cite{kane2017active} are simple in the sense that they are very easy to implement but also do not improve over passive learning beyond $2$-dimensional data.  On the other extreme, mistake-based queries such as conditional-class queries 
\cite{balcan2012robust} and seed queries \cite{bressan2022active}, where the learner selects a set of examples from the dataset and requests an example with a proposed label, 
allow the learner to label the whole dataset with very few queries but are very complicated
in the sense that each one requires transfering a lot of information from the learner to 
the labeler (essentially the learner has to transfer their whole dataset) making them
impractical.  Motivated by those gaps in the literature, we ask the following natural question.


\emph{Can we design simple query classes that simultaneously lead to active learning algorithms with low query complexity?}

\paragraph{Example: 2-d Halfspaces} 
Consider the 2-dimensional halfspace learning problem shown in \Cref{fig: example}. A learner is given a complicated unlabeled dataset $S \subset \R^2$ labeled by some unknown halfspace $h^*$ and wants to learn the labels of examples in $S$. Consider the shadowed region $T$ in \Cref{fig: example}. There is a significant fraction of examples contained in $T$ and all of them have the same label. If one can verify this fact, then a huge progress is made for the learning task. However, if the learner can only use label queries, then to verify this fact, every example in this region has to be queried once. This is why vanilla active learning has a high query complexity. On the other hand, the region $T$ is independent on the dataset $S$.
The structure of $T$ is so simple that to describe $T$ for the labeler, the learner only needs to send information about the two halfspaces that define $T$. Once the labeler describes the region $T$ for the labeler, the labeler can easily respond to the learner and the verification problem can be solved in a single round of simple interaction. This implies that a simple query language may help a lot in learning and motivates the following learning model.



\begin{figure}
    \centering
    \includegraphics[width=0.6\linewidth]{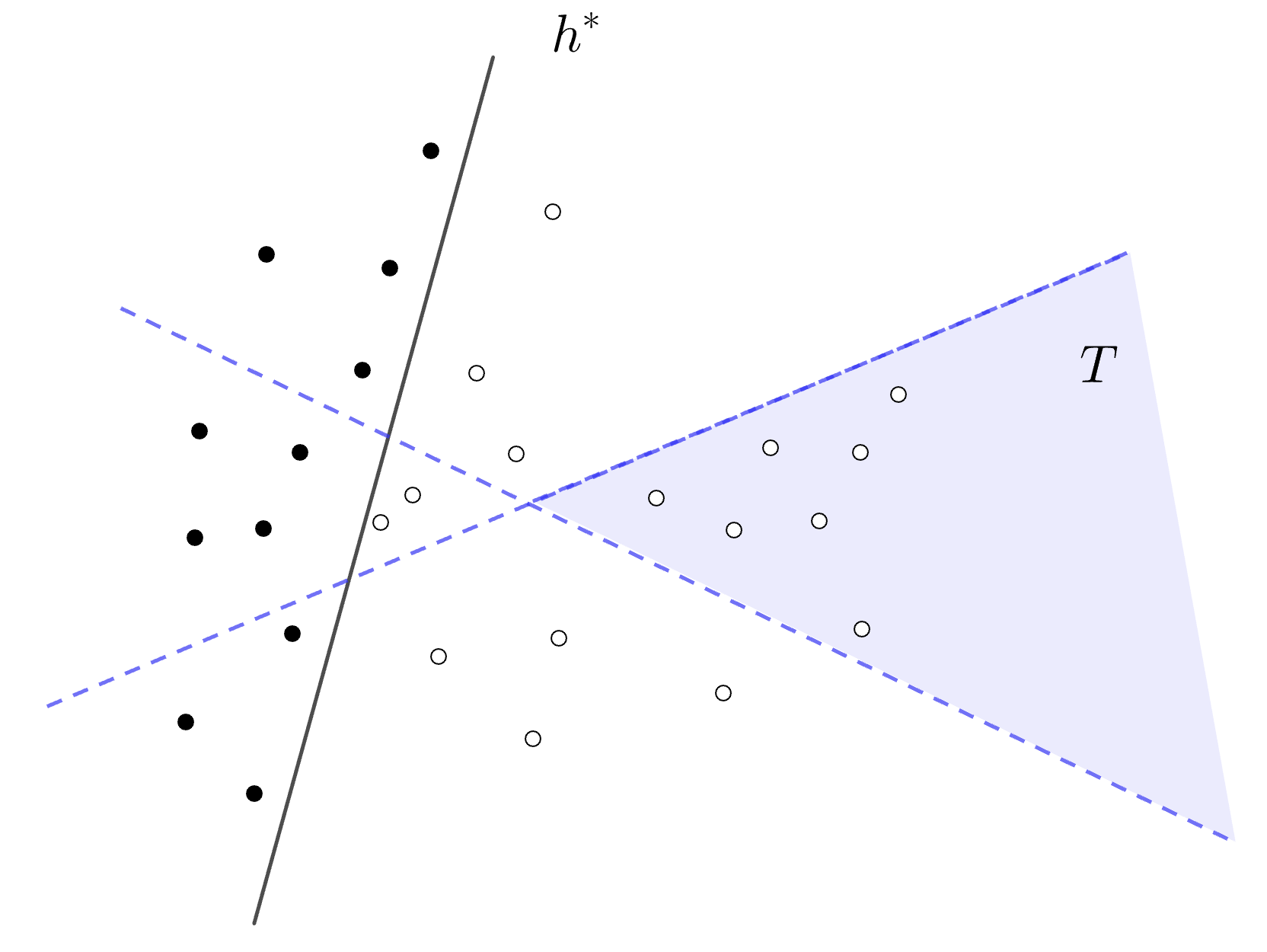}
    \caption{Learning 2-dimensional Halfspaces with Region Queries}
    \label{fig: example}
\end{figure}



\begin{definition}[Active Learning with Region Queries]
Let $\H$ be a class of binary hypotheses over a domain $\X$ and let $h^* \in H$ be a true hypothesis that labels the examples in $\X$ as positive or negative.
Given a set of $n$ examples $S \subseteq \X$, a learner $\A$ wants to learn the labels of examples in $S$ by adaptively submitting \textbf{region queries} to a labeler from a query family $Q$. 
In particular, a region query $q=(T,z) \in Q$ consists of a subset $T$ of $\X$ and a label $z \in \{\pm 1\}$. The labeler has a (possibly unknown) labeling domain $L$ such that $S \subseteq L \subseteq X$ and
after receiving a query $(T,z)$, answers whether all examples in $L \cap T$ have label $z$ under the true hypothesis $h^*$.\footnote{ If $L \cap T = \emptyset,$ then the labeler can output an arbitrary answer.}
\end{definition}
We remark that an important feature of region queries is that the query family $Q$ is defined independently on the dataset $S$. Such an additional requirement not only captures the feature that checking whether an example with a given label exists in a region could be much easier than labeling every example in the region but also captures the features of many other applications such as human learning \cite{shanks1994characteristics}, theorem proving \cite{davis1962machine} and learning via language models \cite{polu2020generative}. 




\paragraph{2-d Halfspaces (cont.)}  We now revisit the previous example where a set of $n$ points $S \subseteq \R^2$ are labeled by some hidden halfspace $h^*$ to illustrate how region queries can be used to efficiently obtain the labels of all examples. It is well-known that \cite{megiddo1985partitioning}, for any set of $n$ points in $\R^2$, one can compute in $O(n)$ time, two lines that partition these points into $4$ regions, each of which contains at least $\lfloor n/4 \rfloor$ points,
see \Cref{fig: example}. We notice that $h^*$ can have at most $2$ intersections with the two lines, which implies at least one of the four regions lies on the one side of $h^*$. Now, if we make region queries over these four regions, then
with at most $8$ region queries, we can identify a region $T$ as in \Cref{fig: example}, which contains only points with the same label and thus label a quarter of $S$. If we repeat this process over the remaining examples $O(\log n)$ rounds, we successfully infer the label of every example with $O(\log n)$ region queries.
In particular, the algorithm used here does not rely on the label of a single example in the dataset to make an update, and every query used by the algorithm is binary.
Furthermore.
the query family $Q=\{(T,z) \mid z \in \{\pm 1\}, T \text{ is an intersection of two halfspaces} \}$ is predetermined before the learner sees the dataset $S$. Thus, no matter how complicated the dataset $S$ is, in a single round of interaction, the learner only needs to describe the two halfspaces to the labeler and let the labeler check the answer to the query. This requires sending just 4 numbers plus a binary label. Motivated by our success in this example we ask:

\emph{Given a hypothesis class $\H$, can we design a region query family $Q$, where the region used in a query comes from a simple set family, such that $O(\log n)$ queries suffice to perfectly label every set of $n$ examples? If this is true, how complicated should the set family be?}














\subsection{Our results}

\paragraph{Characterizing the Complexity of Learning with Region Queries} 
We measure the complexity of the region query class using the VC dimension of the family of regions.
VC dimension characterizes the capacity of a set family and is one of the most well-studied complexity measures in learning theory. Queries from a query family with bounded VC dimension can be communicated using few bits:
for a finite domain $\X$, and a set family $C$ of VC dimension $d$, communicating a set $c \in C$ only requires $O(d\log (\card{\X}))$ bits. 

Our first main result shows that if the hypothesis class $\H$ has a VC dimension $d$, we can always design a simple query family $Q$ with VC dimension at most $O(d)$ and 
use it to perfectly label any set of $n$ examples with $O(d\log n)$ regions queries. Formally, we have the following theorem.


\begin{theorem}\label{th general up}
   Let $\X$ be a space of example and $\H$ be a hypothesis class over $\X$ with VC dimension $d$.  
   There is a region query family $Q$ over $\X$ with VC dimension at most $6d$ and a learning algorithm $\A$ such that for any set of $n$ examples $S \subseteq \X$ labeled by any true hypothesis $h^* \in \H$,
   $\A$ makes $O(d \log n)$ region queries from $Q$ and correctly label every example in $S$, if the labeling domain $L=S$.
\end{theorem}

 In particular, the $O(d\log n)$ query complexity in \Cref{th general up} matches the lower bound for the query complexity of active learning with arbitrary binary-valued queries in \cite{kulkarni1993active} and thus is essentially information-theoretically optimal. Given \Cref{th general up}, an immediate question is in general, whether it is possible to quickly learn $\H$ with an even simpler query class (with VC dimension $o(d)$). Our next main result gives a negative answer firmly. We give a matching lower bound showing that unless the hypothesis class $\H$ has a good structure, a region query family with VC dimension $\Omega(d)$ is necessary to achieve query complexity $O(\log n)$. Formally,  we have the following theorem.

\begin{theorem}\label{th general low}
For every $d \in N^+$ and $n \ge d$ large enough, there exists a space of examples $\X$ and a hypothesis class $\H$ over $\X$ with VC dimension $d$ such that there exists a set of $n$ example $S$ such that
for every region query family $Q$ over $\X$ with $VC\dim(Q) \le (d-2)/3$ and every active learning algorithm $\A$, there exists a true hypothesis $h^* \in \H$, such that if $\A$ makes less than $\poly(n)$ region queries from $Q$, then with probability at least $1/3$, some example $x \in S$ is labeled incorrectly by $\A$. In particular, this even holds when $\A$ knows the labeling domain $L=S$.
\end{theorem}


 \Cref{th general up} and \Cref{th general low} together give a perfect trade-off between the complexity of the query family and the query complexity and thus show that the VC dimension is a good measure for the performance of region queries.
 We want to remark that \Cref{th general low} not only holds in our learning model where queries are binary but also holds in the stronger model studied in \cite{balcan2012robust,bressan2022active}, where a counter-example is also returned in each round of interaction. 
 We also remark that \Cref{th general low} gives an optimal lower bound that matches \Cref{th general up} only in a minimax perspective. In general, it could be the case where for a very special hypothesis class $\H$, we can design a query family with a much smaller VC dimension than the one of $\H$ but still achieve the information-theoretically optimal query complexity. These examples will be shown later. Furthermore, given a pair of hypothesis classes and query class $(\H,Q)$, we actually come up with a combinatorial characterization of the query complexity of learning $\H$ using $Q$. However, since the result is far from the central theme of this paper, we leave it for Appendix~\ref{sec dimension}.

\paragraph{Efficient Learning Algorithms for Natural Hypothesis Classes}

Although \Cref{th general up} gives an algorithm that can perfectly label every subset of $n$ examples with an optimal query complexity, the algorithm itself is not efficient, as it needs to solve optimization problems over the hypothesis class, which is usually exponentially large with respect to the input. In this work, we also focus on designing query families and learning algorithms for some natural hypothesis classes and obtain stronger results. Specifically, our learning algorithms are computationally efficient and work even when queries are not answered based on the dataset $S$ but on any unknown superset $L$ of $S$. These results are summarized as follows.


\begin{theorem}\label{th summary}
    There is a computationally efficient algorithm $\A$ and a query class $Q$ such that for any set $S$ of $n$ examples, $\A$ learns the labels of $S$ perfectly by making region queries to a labeler 
    with labeling domain an unknown set $L \supseteq S$:
    \begin{enumerate}
    \item For unions of $d$ intervals, $Q$ has VC dimension $2$, and $\A$ makes $O(d\log n)$ queries.
    \item For axis parallel boxes in $\R^d$, $Q$ has VC dimension $O(\log d)$, and $\A$ makes  $O(d\log n)$ queries.
    \item For halfspaces in $\R^d$, $Q$ has VC dimension $\Tilde{O}(d^3)$, and $\A$ makes $\Tilde{O}(d^3 \log n)$ queries.
    \end{enumerate}
\end{theorem}

We note that for the first two cases, the VC dimension of the query class
is significantly smaller than the VC dimension of the hypothesis class which is $\Theta(d)$. In the case of halfspaces, the VC dimension of $Q$ and the query complexity is worse than that given in Theorem~\ref{th general up} but applies in a significantly more general setting and is computationally efficient. The cubic dependence on $d$ can be improved to quadratic if the learner provides counter-examples instead of binary answers to our region queries. We leave the detail discussion on this improvement to Appendix~\ref{app halfspace}.


\begin{table}
    \centering
    \begin{tabular}{c|c|c|c|c}
        \textbf{Hypothesis Class} $\H$  & \textbf{VC-dim($\mathrm{Q}$)} & \textbf{Query Complexity}  & \textbf{Efficient?} & \textbf{Labeling Domain} \\ \hline \hline
        General & $O(d)$ & $O(d \log n)$ & No & $L = S$ \\ \hline
        Union of $d$ Intervals & $O(1)$ & $O(d \log n)$ & Yes & $L \supseteq S$ \\ \hline
        Axis Parallel Boxes & $O(\log d)$ &$O(d \log n)$ &Yes & $L \supseteq S$  \\ \hline
        Halfspaces & $\Tilde{O}(d^3)$ & $\Tilde{O}(d^3 \log n)$ & Yes & $L \supseteq S$ \\
    \end{tabular}
    \caption{Summary of the algorithmic results of \Cref{th general up} and \Cref{th summary} for a hypothesis class $\H$ of VC dimension  $\Theta(d)$ and a dataset $S$ of size $n$. }
    \label{tab: result}
\end{table}




\subsection{Connection with Other Learning Models and Related Work}

\paragraph{Active Learning with Enriched Queries}
The study of active learning with enriched queries can be traced to the literature of exact learning \cite{angluin1988queries,balcazar2001general,balcazar2002new,chase2020bounds}. More recently, the focus has been shifted from general queries to more problem-dependent queries such as mistake-based queries \cite{balcan2012robust,bressan2022active}, clustering-based queries \cite{ashtiani2016clustering,mazumdar2017clustering,bressan2021exact,del2022clustering,xia2022optimal}, comparison-based queries \cite{kane2017active,kane2018generalized,xu2017noise,hopkins2020noise,hopkins2020point,hopkins2020power}, separation-based queries \cite{har2021active} and derivative-based queries \cite{ben2022active}. In this work, we study active learning with region queries for both general hypothesis classes and concrete learning problems.



\paragraph{Mistake-Based Query and Self-Directed Learning}
The region queries we study in this paper fall into the category of mistake-based queries \cite{angluin1988queries,maass1992lower,balcan2012robust,bressan2022active}. The study of mistake-based queries can be traced to the study of learning with equivalence or partial equivalence queries \cite{angluin1988queries,maass1992lower}. Though named differently, a typical mistake-based query can be understood as follows. A learner selects a subset of examples $T \subset \X$, proposes a possible labeling for examples in $T$, and submits the information to a labeler. The labeler will return an example $x \in T$ labeled incorrectly by the learner or return nothing when every example in $T$ is labeled correctly. 
We will discuss in Appendix~\ref{app self},
if an arbitrary complicated subset $T$ and any possible labeling are allowed to be used,
a learner could use mistake-based queries to implement online learning algorithms or self-directed learning algorithms \cite{goldman1994power} and easily obtain active learning algorithms with low query complexity.
Our query model has several differences from the previous work. 
(i) Unlike all previous work on mistake-based queries, a region query is a binary query and does not require a counter-example to be returned. 
(ii) Unlike \cite{angluin1988queries,maass1992lower}, a region query is not answered based on the example space $\X$ but based on some labeling domain $S \subseteq L \subseteq \X$ (usually $L=S$). In general, we should not hope to obtain useful information from examples not in the dataset. 
(iii) Unlike \cite{balcan2012robust,bressan2022active}, we require the learner to design a query family $Q$ with finite VC dimension before seeing the dataset $S$ and thus we cannot simply design an active learning algorithm by reducing it to online/self-directed learning.

\paragraph{Learning Halfspace with the Power of Adaptivity}
The class of halfspaces is one of the most well-studied hypothesis classes under active learning. \cite{dasgupta2004analysis} shows that to perfectly learn the labels of a set of $n$ points in $\R^2$ labeled by some halfspace, vanilla active learning needs to make $\Omega(n)$ label queries. Since then, a large body of works \cite{dasgupta2005analysis, balcan2007margin,balcan2013active,balcan2017sample,awasthi2017power} have been done to understand under which distribution vanilla active learning can learn a halfspace with few queries. On the other hand, the query complexity of learning halfspaces in the distribution-free setting is much less understood. \cite{kane2017active} points out that with the help of comparison queries, one can efficiently learn a 2-dimensional halfspace with a query complexity $O(\log n)$. However, in the same work, they point out that such an improvement disappears in $\R^3$. Recently, two remarkable results have been done to understand the query complexity of learning halfspaces in the distributional free setting. The first one is \cite{hopkins2020point}, where they show that if one can query the label of any point in $\R^d$, then $\Tilde{O}(d\log n)$ queries are sufficient to perfectly label $n$ examples. The second one is \cite{bressan2022active}, in which they show without restriction on the complexity of the mistake-based query, they can efficiently learn a $\gamma$-margin halfspace with $\Tilde{O}(d\log(1/\gamma))$ queries. Our halfspace learning algorithm does not rely on acquiring additional information from $\X \setminus S$ or using very complicated query classes but is still able to achieve a query complexity of $\poly(d,\log n)$.

\paragraph{Organization of the Paper}
In \Cref{sec general}, we discuss our results for general hypothesis classes. We give proof sketches for 
\Cref{th general up} and \Cref{th general low} in \Cref{sec general up} and \Cref{sec general low}. In \Cref{sec efficient}, we discuss how to design query classes and efficient active learning algorithms for natural classes. In \Cref{sec interval} and \Cref{sec box}, we study the class of the union of $k$-intervals and the class of high dimensional boxes. In \Cref{sec halfspace}, we discuss our main results on efficient active learning algorithms for halfspaces. Due to the limited space, we present the notations and detailed proofs in the Appendix.


\section{Active Learning for General Hypothesis Class Using Simple Region Query}\label{sec general}


\subsection{Construction of Simple Query Classes for General Hypothesis Classes}\label{sec general up}

In this section, we give an overview of the proof of \Cref{th general up} and leave the full proof and detailed discussion of \Cref{th general up} to Appendix~\ref{app general up}.
Before diving into the proof, we first give an overview of why the previous works on mistake-based queries result in using query families with unbounded query complexity.
Previous work such as \cite{maass1992lower,balcan2012robust} design learning algorithms based on the fact that it is possible to use region queries to implement the Halving algorithm. An algorithm of this style predicts a label for each example in $S$ via majority voting, submits examples with positive predictions, and examples with negative predictions, and gets one example on which majority voting makes a mistake (such a mistake can be found via binary search if the query is binary). In this way, hypotheses that predict incorrectly on this example cannot be the true hypothesis, and the size of the version space is shrunk by half. Since the majority voting could behave arbitrarily complicated over an arbitrary set of examples, the query family used by the algorithm could also be arbitrarily complicated. This suggests a new algorithmic framework should be come up with to break the bottleneck.

The intuition behind our algorithm is as follows. 
Assume the examples in $S$ have been ordered as $x^{(1)},\dots,x^{(n)}$. We consider $H^{(0)}$, the restriction of $\H$ over the dataset $S$. If we make a label query for $x^{(1)}$, then such a label query might not be very helpful because most of the hypotheses in $H^{(0)}$ could label this example in the same way, for example, $y^{(1)} \in \{\pm 1\}$. Let's assume we are in this case and define $H^{(1)}:=\{h \in H^{(0)} \mid h(x^{(1)}) =y^{(1)} \}$. Similarly, a label query for $x^{(2)}$ is also not that useful, since many hypotheses in $H^{(1)}$ might label $x^{(2)}$ by some $y^{(2)}$. Assuming we are in this case, then we have a new class $H^{(2)}$ defined based on $H^{(1)}$. Although each single label query is not useful,
if we repeat this process, at some point $t \in [n]$, the remaining hypothesis class $H^{(t)}$ should have a proper size. i.e $\card{H^{(t)}} /\card{ H^{(0)}} \in (1/3,2/3)$. This implies that after $t$ label queries no matter what answer we get, the size of the version space is shrunk by a constant factor. Notice that these $t$ label queries can be safely replaced by $2$ region queries, $(\{x \mid x = x^{(i)}, i \in [t], h'(x)=1\},1)$ and $(\{x \mid x = x^{(i)}, i \in [t], h'(x)=-1\},-1)$, 
where $h'$ is an arbitrary hypothesis whose restriction over $S$ is in the class $H^{(t)}$. By Sauer's lemma, $\card{H^{(0)}} \le O(n^d)$. So, if we repeat the above procedure $O(d\log n)$ times, we learn the labels of examples in $S$.
Up to now, the problem has been almost solved, but the regions where we make queries still depend on the dataset $S$. However, the analysis above works for any order of $S$, if there is a natural order $o$ for $\X$, then the constraint $x=x^{(i)}$ for some $i \in [t]$ can be simply replaced by $o(x) \in [o(x^{(1)}),o(x^{(t)})]$, because $L=S$. Thanks to the well-known well-ordering theorem, such a linear order exists for every non-empty space $\X$. Thus, we can construct the query family $Q$ using $\H$, $\Bar{\H}$(the set of negation hypothesises in $\H$), and the natural linear ordering defined in $\X$, which gives a simple query class.


\subsection{Lower Bound on the VC Dimension of the Query Class}\label{sec general low}

In this section, we give an overview of the proof of \Cref{th general low}, showing a matching lower bound for \Cref{th general up}. The full proof and more detailed discussions are presented in Appendix~\ref{app general low}. 

We will assume $\X$ to be a space of $n$ examples and $L=S=\X$. i.e. The labeling domain, the dataset to be labeled, and the example space are the same. 
Suppose there is some subset $C^* \subseteq S$ of size $k$ and we want to distinguish hypothesis $h_0$, which labels every example in $C^*$ positive and everything else negative, and the other $k$ hypothesis $h_1,\dots,h_k$, each of which only differs from $h_0$ at a single example in $C^*$. Let's assume the learning algorithm is using a fixed region query class $Q$ to learn. 
For any query $(T,z) \in Q$, if $T$ has an intersection with both $C^*$ and $S \setminus C^*$, then it will provide no useful information(even if some example $x \in T \cap S$ with label $-z$ is also returned), because we know that $q(T,z)=0$ always holds. 
Furthermore, to distinguish the two cases, those regions $T \subseteq C^*$ should cover all examples in $C^*$, otherwise, an example $x \in C^*$ is not involved in any query. As we will show later, the optimal solution to this set cover instance roughly serves as a lower bound of the query complexity in this special instance. In particular, if every $T \subseteq C^*$ as size at most $t$, then the query complexity should be at least $\Omega(k/t)$.

So far, we have established a hard instance for a fixed query class. The most difficult part of our construction is to generalize the above instance so that it is hard for \emph{every} query class $Q$ with VC dimension $O(d)$, where massive subsets of $\X$ would be possible to be involved. We use several key techniques to overcome this difficulty. 
The first one is the following observation. If we have $N>\card{\dom(Q)}$ subsets $C_1,\dots,C_N \subseteq \X$ of size $k$ such that the pairwise intersection of $C_i,C_j$ is at most $t$, then there must be at least one $C_i$ such that if $T \subseteq C_i$ and $T \in \dom(Q)$ then it must be the case $\card{T} \le t$. Sauer's lemma tells us that each set family over $\X$ with VC dimension $O(d)$ contains at most $O(n^d)$ different sets. Thus, if we set up the above $N$ to be $O(n^d)$, then we can embed the hard instance we mentioned above into each $C_i$ so that any learning algorithm uses any query class with VC dimension $O(d)$ has query complexity at least $\Omega(k/t)$. In particular, we will see later, that the hypothesis class we use here has VC dimension $O(t)$. So the final step is to show we can construct these subsets $C_i$ such that $k=\poly(n)$ while $t=O(d)$. 
To show this, we make use of the result in \cite{beideman2014set}, which explicitly constructs set families with low pairwise intersections. This is why intuitively a query family with $\Omega(d)$ VC dimension is also necessary for a query complexity of $O(\log n)$.
We want to remark that the construction of the example space $\X$ in \cref{th general low} is fully combinatorial. So, given any large enough space of examples, we can embed the hard instance we construct in \Cref{th general low} into that space to get a corresponding hard instance. Formally, we have the following corollary, which gives a stronger statement of \Cref{th general low}. We refer the readers to \Cref{app col} for the proof of \Cref{col general low}

\begin{corollary}\label{col general low}
There is a space of examples $\X$ such that for every $d \in N^+$ and $n \ge d$ large enough, there exists a hypothesis class $\H$ over $\X$ with VC dimension $d$ such that there exists a set of $n$ example $S$ such that
for every region query family $Q$ over $\X$ with $VC\dim(Q) \le (d-3)/3$ and every active learning algorithm $\A$, there exists a true hypothesis $h^* \in \H$, such that if $\A$ makes less than $\poly(n)$ region queries from $Q$, then with probability at least $1/3$, some example $x \in S$ is labeled incorrectly by $\A$. In particular, this even holds when $\A$ knows the labeling domain $L=S$.
\end{corollary}

\section{Efficient Active Learning with Simple Questions for Natural Hypothesis Classes}\label{sec efficient}

In \Cref{sec general up}, we have shown that given a hypothesis class $\H$ with dimension $d$, we can construct a query class $Q$ with dimension $O(d)$, so that a learner can use $Q$ to learn $\H$ with query complexity $O(d\log n)$. However, the learning algorithm we use \Cref{sec general up} is not computationally efficient and works when the labeling domain is the same as the dataset. i.e. $L=S$. Such an assumption might be strong for some applications. For example, if a learner is interacting with a large language model, then the language model cannot know the learner's dataset $S$ in advance and thus will answer the learner's query based on an unknown and potentially much larger superset $L$ of $S$.
In this section, we focus on designing learning algorithms with low query complexity for natural hypothesis classes including the union of $k$ intervals, high dimensional boxes, and $d$-dimensional halfspaces, for which the query complexities are $\Omega(n)$ in the vanilla active learning setting. Our algorithms are not only efficient but also work even when the queries are not answered based on the learner's dataset $S$ but on any unknown superset $L$ of $S$. In particular, we will see that when the hypothesis class has a good structure, the query family $Q$ used by our algorithm can have $O(\log d)$ or even a constant VC dimension. Due to space limitations, we leave the full proofs and detailed discussions to Appendix~\ref{app efficient}.


\subsection{Learning Union of $k$ Intervals}\label{sec interval}
The first hypothesis class we study is the class of the union of $k$ intervals, perhaps one of the simplest classes studied in the active learning literature. In the following theorem, we design an efficient learning algorithm that uses $O(k\log n)$ ``interval'' queries to learn a target hypothesis over an arbitrary set of $n$ examples.
\begin{theorem}\label{th k-subset}
    Let $\X=\R$ be the space of examples and $\H=\{h \mid \exists [a_i,b_i], i \in [k], s.t. h(x)=1 \iff x \in \cup_{i=1}^k [a_i,b_i] \}$ be the class of union of $k$ intervals over $\R$. Let $I$ be the class of intervals over $\R$ and query family $Q=\{(T,z) \mid T \in I, z \in \{\pm 1\}\}$. There is a learner $\A$ such that for every subset of $n$ examples $S$, labeled by any $h^* \in \H$ and for every labeling domain $S \subseteq L$(possibly unknown to $\A$), $\A$ runs in $O((T+n)k\log n)$ time, makes $O(k\log n)$ queries from $Q$ and labels every example in $S$ correctly, where $T$ is the running time to implement a single region query.
\end{theorem}
We give the proof overview of \Cref{th k-subset} here and leave the full proof for Appendix~\ref{app interval}. The main idea that we use is that, any $h^* \in \H$ partitions $\R$ into $2k+1$ intervals $I_1,\dots,I_{2k+1}$. Examples in the same interval have the same label, while examples in two adjacent intervals have different labels. So, instead of learning $k$ intervals at the same time, it is sufficient to design a learning algorithm that learns examples in $S$ in the left-most interval. 
This can be done easily using interval queries and binary search. We order $S$ by $x^{(1)}<\dots<x^{(n)}$. Suppose $x^{(1)} \in I_1$ and has label $y=-1$. Then no matter which $L$ the labeler has, using $O(\log n)$ interval queries via binary search, we are able we find $i^*$ such that $q([x^{(1)},x^{(i^*)}],-1)=1$ and $q([x^{(1)},x^{(i^*+1)}],-1)=0$. After this, we can safely label example $x^{(1)},\dots,x^{(i^*)}$ by negative. In particular, examples in $I_1 \cap S$ are all labeled in this iteration because
$I_1 \cap S \subseteq [x^{(1)},x^{(i^*)}] \cap S$. By repeating the procedure $O(k)$ times, we perfectly label $S$.


\subsection{Learning High-Dimensional Boxes}\label{sec box}
Our next result, \Cref{th box}, gives an efficient learning algorithm for learning a high dimensional box with low query complexity. The full proof of \Cref{th box} is presented in Appendix~\ref{app box}.

\begin{theorem}\label{th box}
    Let $\X=\R^d$ be the space of examples and $\H=\{\prod_{i=1}^d [a_i,b_i] \mid a_i,b_i \in [-\infty,\infty]\}$ be the class of axis-parallel boxes in $\R^d$ that labels $\X$. There is a query class $Q$ over $\R^d$ with VC dimension $O(\log d)$ and an efficient algorithm $\A$ such that for every set of $n$ examples $S \subseteq \R^d$, every target hypothesis $h^* \in \H$ and for every labeling domain $S \subseteq L$(possibly unknown to $\A$), $\A$ runs in $O((T+n)d\log n)$ time, makes $O(d\log n)$ queries from $Q$ and labels every example in $S$ correctly, where $T$ is the running time to implement a single region query.
\end{theorem}
The idea behind \Cref{th box} is similar to that of \Cref{th k-subset}. Instead of learning the target box $h^*=\prod_{i=1}^d [a^*_i,b^*_i]$ directly, we learn each boundary $a^*_i,b^*_i$ separately. Let $b^*_i$ be a boundary of $h^*$. Suppose we can learn some $\hat{b}_i \le b^*_i$ such that for every $x \in S$, if $x_i>\hat{b}_i$ then $x$ is labeled by $-1$. Then the box $\hat{h}=\prod_{i=1}^d [\hat{a}_i,\hat{b}_i]$ perfectly label $S$. This is because if an example $x \in S$ is labeled negative by $\hat{h}$, then $x$ must violate one of the constraints of $\hat{h}$ and have true label $-1$. On the other hand, since $\hat{h} \subseteq h^*$, every example labeled positive by $\hat{h}$ must also have true label $+1$. In fact, we can learn such a $\hat{b}_i$ via region queries of the form $(\{x \mid x_i \ge c\}, -1)$. If we order $S$ such that $x^{(1)}_i \le \dots \le x^{(n)}_i$, then we can use binary search with $O(\log n)$ queries to find the $i^*$ such that $q(\{x \mid x_i \ge x^{(i^*)}_i\}, -1)=0$ but $q(\{x \mid x_i \ge x^{(i^*+1)}_i\}, -1)=1$. We will show in Appendix~\ref{app box} that $\hat{b}_i=x^{(i^*)}_i$ is a good approximation of $b^*_i$ that we want no matter which $L$ the labeler uses. This gives the idea of the query complexity in \Cref{th box}. In particular, the query class we use is defined by the set of axis-aligned halfspaces, which has a VC dimension $O(\log d)$.

\subsection{Learning Arbitrary High-Dimensional Halfspaces}\label{sec halfspace}
Our central results for efficient learning are on half-spaced learning problems. Before this work, even assuming the labeling domain $L=S$, there are no known efficient algorithms for the class of halfspaces that can achieve a query complexity of $\poly(d,\log n)$, even using arbitrarily complicated query classes.
Previous work by \cite{bressan2022active}, assumes each example in $S$ has a margin of $\gamma$ with respect to the target $w^*$ and some counter-example $x \in S \cap T$ with label $-y$ is returned if $q(S \cap T,y)=0$. The query complexity of their algorithm is $O(d\log (d/\gamma))$ and could be potentially $\Omega(n)$ if $\gamma$ is very small. 
However, we want to point out that if we are allowed to communicate arbitrary subsets of the dataset $S$, then by reducing the active learning problem to self-directed learning\cite{goldman1994power}, it is easy to design an efficient halfspace learning algorithm with an expected query complexity $O(d\log^2 n)$ using the idea of \cite{haussler1994predicting} on binary prediction over random points. We summarize the discussion as the following theorem and leave the full proof and more detailed discussion for Appendix~\ref{app self}.

\begin{theorem}\label{th SVM}
Let $\X = \R^d$ be the space of examples and $\H = \{w \mid w \in S^{d-1}\}$ be the class of homogeneous halfspaces in $\R^d$ that labels $\X$. Let $Q=\{(T,z) \mid z \in \{\pm 1\},T\subseteq 2^{\R^d}\}$ over $\R^d$ be the query class that contains any subset of $\R^d$. There is an efficient algorithm $\A$ such that for every set of $n$ examples $S$, labeled by any $w^* \in \H$ and for every labeling domain $S \subseteq L$ (possibly unknown to $\A$), $\A$ runs in $O((T+B)d\log^2 n)$ time, makes $O(d\log^2 n)$ queries from $Q$ in expectation and labels every example in $S$ correctly, where $T$ is the running time to implement a single  query and $B$ is the bit complexity of $S$.
\end{theorem}

Although efficiently learning a halfspace with an arbitrarily complicated query class is easy, designing an efficient learning algorithm using a query class with low VC dimension is significantly more challenging, especially when a query $(T,z)$ is answered based on an unknown superset $L$ of $S$. 

There are several difficulties with this problem.
First, as $(T,z)$ is checked over $L \supseteq S$, there is no way to find an example $x \in S$ with label $-z$, when $q(T,z)=0$. It could be the case that every example in $T \cap S$ has label $z$ but some hidden $x \in T \setminus S$ with label $-z$ makes $q(T,z)=0$. Such difficulty makes it very hard to learn from mistakes without sending the whole dataset to the labeler, which results in a very complicated query family. The second difficulty is how to design the query class so that we can get enough information from a single query. As $L$ is unknown to the learner if a region $T$ is too large, it is very likely that $T$ contains both positive examples and negative examples in $L$, and such queries $(T,z)$ may always return $0$ to the learner, sending no information. On the other hand, if a region $T$ is very small, then each query can only send us very little information because if $L=S$, each query can only provide information about very few examples in $S$. We overcome the above difficulties and obtain the following theorem.


\begin{theorem}\label{th halfspace}
    Let $\X=\R^d$ be a space of examples and let $\H = \{w \mid w \in S^{d-1}\}$ be the class of homogeneous halfspaces in $\R^d$ that labels $\X$.
    There is an algorithm and a query class $Q$ with VC dimension $\Tilde{O}(d^3)$ such that for every subset of $n$ examples $S \subseteq \X$, every labeling domain $L$ with $S \subseteq L$ and every target hypothesis $w^* \in \H$ and every $\alpha \in (0,1)$, it in expectation makes $O(d^3\log^2d \log (1/\alpha))$ queries from $Q$, runs in $\poly(d,n,T)$ time and labels $(1-\alpha)$ fraction of examples in $S$ correctly, where $T$ is the running time of implement a single region query from $Q$. 
    In particular, the algorithm makes $O(d^3\log^2d\log n)$ queries from $Q$ and labels every example in $S$ correctly in time $\poly(d,n,T)$.
\end{theorem}
We want to remark that the query class $Q$ we use has a VC dimension $\Tilde{O}(d^3)$. Such a dependence could be improved to $\Tilde{O}(d^2)$, if an example $x \in T \cap L$ with label $-z$ is also returned when $q(T,z)=0$. For a more detailed discussion, we point the reader to Appendix~\ref{app main}.

A particularly surprising part of our result is that if we only want to perfectly $(1-\alpha)$ fraction of the examples in $S$, then the query complexity of our algorithm even does not depend on the size of $S$.  
We present the full proof of \Cref{th halfspace} in Appendix~\ref{app main} and give the intuition of why it is possible to get such a result. We start by assuming our dataset $S \subseteq S^{d-1}$ has the following nice property. For every $w \in S^{d-1}$, $\beta$-fraction of the examples in $S$ have margin $\gamma$ with respect to $w$. i.e. $|w \cdot x| \ge \gamma$. We create an $\gamma/2$-cover, $\mathcal{N}=\{u_1,\dots,u_\ell\}$ for $S^{d-1}$ and associate a ball $B(u_i)$ with radius $\gamma/2$ for each $u_i$. Then each $x \in S$ must belong to some $B(u_i)$. Furthermore, if example $x \in B(u_i)$ has margin $\gamma$ with respect to $w^*$, then every point inside $B(u_i)$ has the same label as $x$. Since $\beta$ fraction of the examples in $S$ have $\gamma$ margin with respect to $w^*$, if we make 2 queries for each $B(u_i)$, then we can safely label at least $\beta n$ examples in $S$. So, if this margin assumption recursively holds after we remove examples we have labeled, we can repeat such a procedure $O(\log n)$ times and finally perfectly label every example in $S$. 
However, such an intuition does not directly lead to efficient learning algorithms. There are two issues we need to overcome.
First, the above margin assumption in general can not be satisfied recursively and sometimes is even not satisfied by the original dataset $S$. Second, even if $\gamma=1/\poly(d)$, $(1/\gamma)^{O(d)}$ queries have to be made each round, due to the large size of $\mathcal{N}$, which is not computationally efficient. We now give a sketch of how to address these two issues.

The first issue can be overcome with Forster's Transform \cite{forster2002linear}. Roughly speaking, given any set of $n$ examples $S \subseteq \R^d$, Forster's transform finds a subspace $V$ of dimension $k$ containing at least $k/d$ fraction of examples in $S$ and a matrix $A$ such that $f_A(S \cap V) = \{f_A(x):=Ax/\norm{Ax} \mid x \in S \cap V\}$ satisfies the above margin assumption with $\gamma=1/(2\sqrt{k})$ and $\beta=1/(4k)$. In particular, \cite{diakonikolas2021forster,diakonikolas2023strongly} shows that given any set of $n$ examples $S$, we can compute in polynomial time a Forster's transform for $S$. This gives us a way to recursively find a large fraction of examples that satisfy the margin assumption and solve the first issue. So, for now, we assume $S$ satisfies the margin assumption with $\gamma=1/(2\sqrt{d})$ and $\beta=1/(4d)$.

The technique we use to overcome the second issue is inspired by the modified perceptron algorithm used by \cite{blum1998polynomial}. Instead of creating a cover for $S$ and doing a brute-force search, we will use queries to implement the modified perceptron algorithm to learn a halfspace $\hat{w}$ that can correctly label every example that has a large margin with respect to $\hat{w}$.
The modified perceptron algorithm works as follows, it maintains a hypothesis $w_t$ and makes an update $w_{t+1} = w_t-x_t(x_t \cdot w_t)$ if $x_t$ is a point that is misclassified by $w_t$. Furthermore, if every $x_t$ we use for an update has a margin $\Omega(1/\sqrt{d})$ with respect to $w_t$, then after $O(d\log d)$ updates, each example with a margin $\Omega(1/\sqrt{d})$ with respect to $w_t$ is correctly classified by $w_t$. As we mentioned previously, finding such an example where we make a mistake is hard. However, we will show that using an $x_t$ such that $(x_t\cdot w_t)(x_t\cdot w^*) \le 1/\poly(d)$ to make an update is enough to achieve the same guarantee. In particular, such an $x_t$ can be found using binary search together with $O(d\log d)$ region queries that are defined by $O(d)$ linear inequalities. To see why this is true, consider the region $T=\{x \mid v_t \cdot x \ge 1/(2\sqrt{d})\}$, where $v_t$ is the unit vector parallel to $w_t$. According to the margin assumption, a large fraction of the examples are contained in $T$. So if $q(T,1)=1$, we safely label a lot of examples correctly. Otherwise, there is at least one point in $T$(not necessarily in $S$) that is misclassified by $w_t$ and if we find such a point we can use it to make a perceptron update. In this case, we partition the region $T$ into small strips $T_i:=\{x \mid v_t \cdot x \in [a_i,b_i]\}$, where $\abs{b_i-a_i} \le 1/\poly(d)$. With binary search, we can use $O(\log d)$ queries to find one such $T_i$ that contains one point that is misclassified by $w_t$. Now, denote by $u_1,\dots,u_{d-1}$ a standard basis of the subspace orthogonal to $w_t$. Using the same binary search approach over $T_i$ for each direction $u_i$, we will finally find a small box $B \subseteq T_i$ with diameter $1/\poly(d)$ that contains at least one point that is misclassified by $w_t$. Since $B$ has a diameter only $1/\poly(d)$, this implies that each point $x_t$ in $B$ is very close to the decision boundary and satisfies $(x_t\cdot w_t)(x_t\cdot w^*) \le 1/\poly(d)$. So, we can choose any point in $B$ to make a perceptron update and after doing this $O(d\log d)$ rounds, we learn a $w_t$ that safely classifies many examples correctly. We remark that there is still a small issue in the above analysis. Since $L$ is unknown to our algorithm, it could be the case during the binary search a region $Z$ we query has an empty intersection with $L$, and an undesirable answer is returned. This issue can be overcome with the following trick. We first query the label of an example $x \in T \cap S$. If $x$ is misclassified by $w_t$, we immediately make an update. Otherwise, every time we make a query $(Z,y)$, we can instead query $(Z \cup \{x\},y)$, which prevents us from querying an empty region and does not make a query more complicated.

So far, we have given an overview of why $\Tilde{O}(d^3\log(1/\alpha))$ queries suffice to correctly label $1-\alpha$ fraction of examples in $S$. Finally, it remains to bound the VC dimension of the query class we use. Recall that the modified perceptron algorithm we used is implemented over the space under the transform $f_A(\cdot)$. As we will discuss in Appendix~\ref{app main} since the target hypothesis is a halfspace, the labels of points are preserved by Forster's transform. So, every time we make a query $(Z,y)$ in the modified perceptron algorithm, the actual query we should make is $(\{x \in V \mid f_A(x) \in Z\},y)$. As we discussed above, $Z$ is a set of $O(d)$ linear inequalities. So, the query class we use is defined by $O(d)$ degree-2 polynomial inequalities, which has VC dimension $\Tilde{O}(d^3)$.


\section{Conclusion and Future Directions}
The fast development of machine learning has not only resulted in many real applications but has also changed the learning paradigm itself. The success of foundation models makes it easier and faster for the learner to get feedback for more complicated questions, turning the learning paradigm from passively learning from labeled data to actively learning from interactions. 
In this work, we initiate the study of active learning with region queries, a specific type of such interaction. We summarize our contribution and list several interesting future directions as open questions.

 An important novelty of this work is using the VC dimension as a measure of the complexity of queries. As we show in the paper, when the learner and the expert share the dataset $S$, the VC dimension gives a good tradeoff between the complexity of the query class and the query complexity of the learning algorithms.
\emph{ Can VC dimension be used to measure the complexity of other learning problems that involve interaction and communication such as distributed learning \cite{balcan2012distributed,kane2019communication}?
} We think this would be an interesting direction to investigate.

To actively learn a hypothesis class $\H$ with $O(\log n)$ queries, a query class with VC dimension $O(d)$ is enough. On the other hand, we have also seen that for some hypothesis classes with good structure, we can learn it with a query class with VC dimension $O(\log d)$ or even $O(1)$. It is natural to ask \emph{which hypothesis class can be learned with a query class with $o(d)$ VC dimension?
} Studying the query complexity of active learning algorithms using a fixed query class would be also an interesting direction.

For several natural hypotheses classes, we design simple query classes and efficient learning algorithms. Surprisingly, these learning algorithms even work when the dataset is not shared between the learner and the labeler. \emph{Does such a phenomenon hold for general hypothesis classes?
} It is important to understand such a question since the assumption that the learner and the labeler share the knowledge of $S$ does not always hold for some real applications.

 Another important direction is learning with noisy queries. In this paper, we only study the realizable cases, assuming each query is answered correctly. \emph{Can we design learning algorithms robust to wrong answers in their queries?}



\section*{Acknowledgments}

This work was supported by the NSF Award CCF-2144298
(CAREER).

\bibliography{mydb}
\bibliographystyle{alpha}

\newpage

\appendix

\section{Notation and Preliminaries}

Let $\X$ be an example space. A hypothesis class $\H$ is a set of binary functions $h: \X \to \{\pm 1\}$. A hidden true hypothesis $h^* \in \H$ assigns a positive or negative label $y(x)=h^*(x)$ to each $x \in \X$.

A region query is a pair $(T,z)$, where $T \subseteq \X$ is a region in $\X$ and $z \in \{\pm 1\}$ is a proposed label. A region query family $Q$ is a set of region queries. We will define $\dom(Q):=\{T \mid (T,z) \in Q, z \in \{\pm 1\} \}$ the set of regions used in a query in $Q$. 
The complexity of a query family is defined by the VC dimension of the set family that $Q$ uses. 

\begin{definition}[VC Dimension of A Query Class]
Let $\X$ be a space of example and $C \subseteq 2^\X$ be a set family over $\X$. The VC dimension $VC\dim(C)$ of $C$ is defined as the largest number $d$ such that there exists a set $S$ of $d$ examples such that $\card{\{c \cap S \mid c \in C\}} = 2^d$. Let $Q$ be a family of region query family $Q$ over $S$. The VC dimension of $Q$ is defined as 
\begin{align*}
    VC\dim(Q): = VC\dim(\{T \mid (T,z) \in Q, z \in \{\pm 1\} \}) = VC\dim(\dom(Q)).
\end{align*}
\end{definition}

A learning process is a sequence of interactions between a learning algorithm $\A$ and a labeler. The learning algorithm $\A$ is given the hypothesis class $\H$, a dataset $S \subseteq X$ of $n$ examples, and a region query family $Q$. The labeler is given a labeling domain $L$ such that $S \subseteq L$. In a single round of interaction, the learning algorithm $\A$ submits a query $(T,z)$ to the labeler based on any information $\A$ received so far. The labeler returns an answer $q(T,z) \in \{0,1\}$ of the query to $\A$. Here, $q(T,z)=1$ if $\forall x\in T \cap L$, $y(x)=z$. In particular, if $T \cap L = \emptyset$, an arbitrary answer can be returned by the labeler.
At the end of the learning process, the learning algorithm outputs a set of labeled examples $O=\{(x,\hat{y}(x)) \mid x \in S' \subseteq S\}$. For $\alpha \in [0,1)$, we say $\A$ labels $1-\alpha$ fraction of $S$ if $\card{O} \ge (1-\alpha) n$ and for each $(x,\hat{y}(x)) \in O$, $\hat{y}(x) = y(x)$. In particular, if $\alpha=0$, we say $\A$ perfectly labels $S$.


\paragraph{Some Facts on VC Dimension}
We list some properties of VC dimension that will be frequently used during the proof.

(i) Let $C_1$, $C_2$ be two set families over a space of examples $\X$ such that $VC\dim(C_1)=d_1$ and $VC\dim(C_2)=d_2$. Then $VC\dim (C_1 \cup C_2) \le d_1+d_2+1$.

(ii) Let $C$ be a set family over a space of examples $\X$ such that $VC\dim(C)=d$. The $k$-fold unions of $C$ and $k$-fold intersections of $C$ is defined as 
\begin{align*}
    C^{k\cup}  :=\{\cup_{i=1}^k c_i \mid c_i \in C\}, C^{k\cap}  :=\{\cap_{i=1}^k c_i \mid c_i \in C\}.
\end{align*}
Then $VC\dim(C^{k\cup}) \le O(dk\log k)$ and $VC\dim(C^{k\cap}) \le O(dk\log k)$.


\section{Missing Details in \Cref{sec general}}\label{app general}

\subsection{Proof of \Cref{th general up}} \label{app general up}

In this section, we prove \Cref{th general up}, which shows every hypothesis class with VC dimension $d$ can be learned with a query class with VC dimension $O(d)$ with an information-theoretic optimal query complexity. To remind the reader, we restate \Cref{th general up} here.

\begin{theorem}(Restatement of \Cref{th general up})
   Let $\X$ be a space of example and $\H$ be a hypothesis class over $\X$ with VC dimension $d$.  
   There is a region query family $Q$ over $\X$ with VC dimension $O(d)$ and a learning algorithm $\A$ such that for any set of $n$ examples $S \subseteq \X$ labeled by any true hypothesis $h^* \in \H$,
   $\A$ makes $O(d \log n)$ region queries from $Q$ and correctly label every example in $S$, if the labeling domain $L=S$.
\end{theorem}

To start with, we will remind the reader of some basic background in set theory. 


\begin{definition}[Strcit Total Order]
    Let $\X$ be a non-empty set. A binary relation $``<"$ over $\X$ is a strict total order if for every $a,b,c \in \X$, the following conditions are satisfied.
    \begin{itemize}
        \item Not $a<a$. (irreflexive)
        \item If $a<b$, then not $b<a$.(asymmetric)
        \item If $a<b$, $b<c$, then $a<c$. (transitive)
        \item If $a \neq b$, then $a<b$ or $b<a$. (connected)
    \end{itemize}
\end{definition}
Consider a set $\X$ with a strict total order $``<"$, we have the following lemma. 
\begin{lemma}\label{lm total order}
    Let $\X$ be a set and $``<"$ be a strict total order over $X$. Let $I=\{[a,b] \mid a,b \in X\}$, where $x \in [a,b]$ if $a \le x \le b$. $VC(I) \le 2$.
\end{lemma}

\begin{proof}(Proof of \Cref{lm total order})
    Let $a,b,c$ be any 3 distinct points in $\X$ such that $a<b<c$. Since $``<"$ is a strict total order, we know that 3 distinct points can be ordered in the above way.
    Let $h=[l,r] \in H$ be any set such that $a \in h$ and $c \in h$. By transitive property, we know that $l \le a < b <c  \le r$, which implies that $b \in h$. Thus, no hypothesis in $I$ can label $a,c$ positive but $b$ negative, which implies $VC(I) \le 2$.
\end{proof}
\Cref{lm total order} implies that if a space of examples $\X$ admits a strict total order, then we are able to define the class of intervals over $\X$, which has a very small VC dimension. If $\X$ is finite, such a strict total order can be easily defined by any permutation of $\X$. If $\X$ is infinite or continuous, we next explain that such a strict total order(linear order) can also be defined. This fact follows the following well-known well-ordering theorem (equivalent to Zorn's lemma and axiom of choice).

\begin{theorem}[well-ordering theorem]\label{th well order}
    A set $\X$ is well-ordered by some strict total order if every non-empty subset of $\X$ has a least element under the order. Furthermore, every set $\X$ can be well ordered.
\end{theorem}

According to \cite{pincus1997dense}, well-ordering theorem implies that every example space admits a strict total order.

With the background of the basic set theory, we are able to prove the following structural result.
\begin{lemma}\label{lm set construct}
Let $\X$ be a space of examples and $``<"$ be a strict total ordering over $\X$. Let $\H$ be any hypothesis class over $\X$.
Let $S \subseteq \X$ be any subset of $n$ examples. Define $H_S=\{h_S: S \to \{\pm 1\} \mid \exists h \in \H, h_S(x)=h(x), \forall x \in S\}$ be the hypothesis class of $\H$ restricted at set $S$. If $\card{H_S} > 1$, then there exists an interval $[a,b]$ and a hypothesis $h$ such that $\card{\{h_S \in H_S \mid h_S(x)=h(x), \forall x \in [a,b] \cap S\}} \in [\card{H_S}/3,2\card{H_S}/3]$.  
\end{lemma}

\begin{proof}(Proof of \Cref{lm set construct})
    We order examples in $S$ according to the strict total order $``<"$ and denote by $x^{(1)}<x^{(2)}<\dots<x^{(n)}$ these ordered examples. Given this ordered dataset $S$, we recursively define $i$th majority prediction class $M^{(i)}$ in the following way, $M^{0} = H_S,$
    \begin{align*}
     M^{(i+1)} &= \{h_S \in M^{(i)} \mid \card{M^{(i)} \cap \{h' \in H_S \mid h'(x^{(i+1)}) = h_S (x^{(i+1)})\} } \ge \card{M^{(i)}}/2 \}. 
    \end{align*}
     That is to say, $M^{(i+1)}$ is the class of hypothesis in $M^{(i)}$ that predicts the label of $x^{(i+1)}$ according to the majority of $M^{(i)}$. Let $i^* \in [n]$ be the smallest number such that $\card{M^{(i^*)}} \le 2\card{H_S}/3$. We notice that $\card{M^{(i^*)}} \ge \card{H_S}/3$ because
    \begin{align*}
        \card{M^{(i^*)}} \ge \card{M^{(i^*-1)}}/2 > \card{H_S}/3,
    \end{align*}
by the definition of the majority prediction class and $i^*$. Next, we show that such an $i^*$ exists. Notice that $M^{(n)}$ contains a single hypothesis in $H_S$, thus, we have $1=\card{M^{(n)}} \le \card{H_S}/\card{H_S} \le \card{H_S}/2 < 2\card{H_S}/3$. Furthermore, since $\card {M^{(0)}} = \card{H_S}$, we know that $i^* \in [n]$ exists. Now we set $a=x^{(1)},b=x^{(i^*)}$ and $h \in H$ be any hypothesis such that $\exists h_S \in M^{(i^*)}$ agrees with $h$ for every example in $S$. Then we have 
\begin{align*}
    \card{\{h_S \in H_S \mid h_S(x)=h(x), \forall x \in [a,b] \cap S\}} = \card{M^{(i^*)}} \in [\card{H_S}/3,2\card{H_S}/3],
\end{align*}
since $M^{(i^*)} = \{h_S \in H_S \mid h_S(x)=h(x), \forall x \in [a,b] \cap S\}$.

\end{proof}

Given the above structural result, we present \Cref{alg general}, the algorithm we use in the proof of \Cref{th general up}.

	\begin{algorithm}
		\caption{\textsc{GeneralQuery}$(S,\H,Q)$ (Label $S$ with query set $Q$ given hypothesis class $\H$ )}\label{alg general}
		\begin{algorithmic} 
  
\State Let $H_S=\{h_S \mid \exists h \in \H, h_S(x)=h(x), \forall x \in S\}$.
\While {$\card{H_S}>1$}
\State Find interval $[a,b] \in 2^\X$ and $\hat{h} \in \H$ that satisfies the property in the statement of \Cref{lm set construct}.
\State Let $S^+ = \{x \in [a,b]\mid  \hat{h}(x)=1\}$ and $S^-= \{x \in [a,b]\mid  \hat{h}(x)=-1\}$
\State Make query $(S^+,1)$ and $(S^-,-1)$.
\If{$q(S^+,1)=q(S^-,-1)=1$}
\State $\H \gets \{h \in \H \mid h(x)=\hat{h}(x), \forall x \in S\cap [a,b]\}$.
\Else 
\State $\H \gets \H \setminus \{h \in \H \mid h(x)=\hat{h}(x), \forall x \in S\cap [a,b]\}$.
\EndIf
\State $H_S=\{h_S \mid \exists h \in \H, h_S(x)=h(x), \forall x \in S\}$
\EndWhile
\State Label $S$ according to the single partial hypothesis in $H_S$.
		\end{algorithmic}
	\end{algorithm}

\begin{proof}(Proof of \Cref{th general up})
We show that \Cref{alg general} uses a query class $Q$ with VC dimension $O(d)$ that labels $S$ correctly with $O(d\log n)$ queries. 

We first show the correctness of the algorithm. Let $h^*_S$ be the target hypothesis restricted at $S$. Every time we make queries $(S^+,1),(S^-,-1)$ during the execution of \Cref{alg general}, $h^*_S$ agrees with $\hat{h}$ at every example in $S \cap [a,b]$ if and only if $q(S^+,1)=q(S^-,-1)=1$, which implies that $h^*_S$ is always contained in $H_S$. So, at the end of \Cref{alg general}, every example in $S$ is labeled according to $h^*_S$ and thus is labeled correctly.

Next, we bound the number of queries used by the algorithm. According to \Cref{lm set construct}, we know that every time we find an interval $[a,b]$ and $\hat{h}$, we have 
\begin{align*}
    \card{\{h_S \in H_S \mid h_S(x)=\hat{h}(x), \forall x \in [a,b] \cap S\}} \in [\card{H_S}/3,2\card{H_S}/3].
\end{align*} This implies 
\begin{align*}
\card{H_S \setminus \{h_S \in H_S \mid h_S(x)=\hat{h}(x), \forall x \in [a,b] \cap S\}} \le 2\card{H_S}/3.    
\end{align*}
 So, whether $h^*_S$ agrees with $\hat{h}$ over $S \cap [a,b]$ or not, after each update the size of $H_S$ will always shrink by a factor of $2/3$. Since $H$ has a VC dimension of $d$, we know from Sauer's lemma that $\card{H_S} \le O(n^d)$ at the beginning of the execution of \Cref{alg general}. Thus, after $O(d\log n)$ updates $\card{H_S}=1$ and \Cref{alg general} will terminate. The total number of queries is $O(d\log n)$ since we make 2 queries for a single update.

Finally, we upper bound the VC dimension of the query class $Q$ that \Cref{alg general} uses. Notice that $Q=\{[a,b] \cap \{x \mid g(x)=1\} \mid a,b \in X, g \in H \cup \Bar{H}\}$, where $\Bar{H} = \{-h(x) \mid h \in H\}$ is the set of complement of hypothesis in $H$. By the property of VC dimension, we have 
\begin{align*}
    VC(Q) \le VC(I)VC(H \cup \Bar{H}) \le VC(I)(2VC(H)+1) \le 2(2VC(H)+1) \le 6d.
\end{align*}

\end{proof}

\subsection{Proof of \Cref{th general low}}\label{app general low}

In this section, we present the proof of \Cref{th general low}, showing a matching lower bound for \Cref{th general up}. Here, we restate \Cref{th general low} as a reminder.

\begin{theorem}(Restatement of \Cref{th general low})
For every $d \in N^+$ and $n \ge d$ large enough, there exists a space of examples $\X$ and a hypothesis class $\H$ over $\X$ with VC dimension $d$ such that there exists a set of $n$ example $S$ such that
for every region query family $Q$ over $\X$ with $VC\dim(Q) \le (d-2)/3$ and every active learning algorithm $\A$, there exists a true hypothesis $h^* \in \H$, such that if $\A$ makes less than $\poly(n)$ region queries from $Q$, then with probability at least $1/3$, some example $x \in S$ is labeled incorrectly by $\A$. In particular, this even holds when $\A$ knows the labeling domain $L=S$.
\end{theorem}

We start with \Cref{lm cover}, showing how to construct a hard instance for a fixed query family.

\begin{lemma}\label{lm cover}
    Let $\X$ be a space of examples and let $Q$ be a region query class over $\X$.
    Let $C^* \subseteq \X$ be a set of $k$ examples. Let $H_{C^*}=\{h_S \mid S \subseteq C^*, \card{S} \le 1\}$ be a hypothesis class over $\X$, where $h_S(x)=1$ if and only if $x \in C^* \setminus S$. Assuming for every $T \in \dom(Q)$, if $T \subseteq C^*$, then $\card{T} \le t$. Then for every learner $\A$ that makes $k/2t$ queries from $Q$, there is some hypothesis $h^* \in H_{C^*}$ such that with probability at least $1/3$, there exists some $x \in C^*$ that is mislabeled by $\A$ assuming the labeling domain is the same as the example space i.e $L=\X$.
\end{lemma}

\begin{proof}(Proof of \Cref{lm cover})
    Let $x \in C^*$ be an example. We say $x$ is covered by some query $T \in \dom(Q)$ if either $x \in T \subseteq C^*$ or $T \cap C^* = \{x\}$. Assume that the target hypothesis $h^*$ is drawn uniformly from $H_{C^*}$. Denote by $Q_S \subseteq Q$ the random subset of queries that $\A$ makes and $\hat{h}_S$ the output hypothesis by $\A$ conditioned on the target hypothesis is $h^*=h_S$. Notice that if $x$ is not covered by $\dom(Q_\emptyset)$, then we must have $\hat{h}_{\{x\}} = \hat{h}_\emptyset$. This is because no matter whether the target hypothesis is $h_\emptyset$ or $\hat{h}_{\{x\}}$, each query $\A$ made so far will have exactly the same answer. Specifically, let $(T,z)\in Q_\emptyset$ be any query used by $\A$ so far. We know that $x \not \in T$. If $T \subseteq C^*$, then $T$ only contains positive examples. If $T \subseteq \X \setminus C^*$, then $T$ contains only negative examples. Otherwise, $T$ contains at least one positive example and one negative example.
    
    Now, assuming $\Pr(\hat{h}_\emptyset \neq h_\emptyset) \le 1/3$, we will show there must be some $x \in C^*$ such that $\Pr(\hat{h}_{\{x\}} \neq h_{\{x\}}) > 1/3$. This will follow the standard way of bounding the probability of making an error used in the active learning literature such as \cite{hanneke2015minimax}.
    \begin{align*}
            \max_{x \in C^*} \Pr(\hat{h}_{\{x\}} \neq h_{\{x\}}) & \ge \frac{1}{k} \sum_{x \in C^*}\Pr(\hat{h}_{\{x\}} \neq h_{\{x\}}) \ge \frac{1}{k} \sum_{x \in C^*}\Pr(\hat{h}_{\{x\}} (x)= h_\emptyset (x)) = \frac{1}{k} \E\sum_{x \in C^*} 1_{\{\hat{h}_{\{x\}} (x)= h_\emptyset (x)\}}\\
        & \ge \frac{1}{k} \E\sum_{x \in C^*} 1_{\{x \text{ not covered by }  Q_\emptyset\}} 1_{\{\hat{h}_{\{x\}} (x)= h_\emptyset (x)\}} = \frac{1}{k} \E\sum_{x \in C^*} 1_{\{x \text{ not covered by }  Q_\emptyset\}} 1_{\{\hat{h}_{\emptyset} (x)= h_\emptyset (x)\}} \\
        & \ge \frac{1}{k} \E\sum_{x \in C^*} 1_{\{x \text{ not covered by }  Q_\emptyset\}} 1_{\{\hat{h}_{\emptyset} (x)= h_\emptyset (x)\}} 
         \ge \frac{1}{k} \E 1_{\{\hat{h}_{\emptyset} = h_\emptyset \}} \sum_{x \in C^*} 1_{\{x \text{ not covered by }  Q_\emptyset\}}  \\
        & \ge \frac{1}{k} \Pr(\hat{h}_{\emptyset} = h_{\emptyset}) (k-k/2) > 1/3.
    \end{align*}
So, we conclude that ant learner $\A$ that makes $k/2t$ queries from $Q$ will with probability at least $1/3$ label at least one example in $C^*$ incorrectly.    
\end{proof}

Next, we present \Cref{lm inclusion}, which gives a way to extend the hard instance we constructed in \Cref{lm cover} for a single query class to multiple query classes.

\begin{lemma}\label{lm inclusion}
    Let $\X$ be any space of $n$ examples and $Q$ be a query class over $\X$. Let $\{C_1,\dots,C_N\} \subseteq 2^\X$ 
    be a collection of $N>\card{\dom(Q)}$ subsets over $\X$ such that for every $i,j \in [N]$, $i \neq j$,
    $\card{C_i \cap C_j}\le  t$. There is some $C^* \in \{C_1,\dots,C_N\}$ such that for every $T\in \dom(Q) $ if $T \subseteq C^*$, $\card{T} \le t$.
\end{lemma}

\begin{proof}(Proof of \Cref{lm inclusion})
    We say a query $T_i \in \dom(Q)$ witnesses a set $C_i \in \{C_1,\dots,C_N\}$ if $T_i \subseteq C_i$ and $\card{T_i}>t$. Let $T \in \dom(Q)$ be any region such that $\card{T}>t$, we claim that $T$ can witness at most one set $C$ from $\{C_1,\dots,C_N\}$. This is because if there exists $C_i,C_j \in \{C_1,\dots,C_N\}$, $i \neq j$, that are witnessed by $T$, then $T \subseteq C_i \cap C_j$, which implies that $\card{T} \le \card{C_i \cap C_j} \le t$ and gives a contradiction. Since $N> \card{\dom(Q)}$, we know that there must be at least one $C^* \in \{C_1,\dots,C_N\}$ that is not witnessed by any $T \in \dom(Q)$. Thus for every $T \in \dom(Q)$, if $T \subseteq C^*$ then we must have $\card{T} \le t$.
\end{proof}

Besides the above two technical lemmas, we will make use of the following results that construct set families with small pairwise intersections.

\begin{theorem}[Theorem 3 in \cite{beideman2014set}]\label{th intersection}
    For every positive integer $k \ge \gamma$, there exist $N \ge (2k\ln 2k)^{\gamma+1}$ subsets $S_1,\dots,S_N \subseteq [4k^2 \ln 4k]$ such that for every $i \neq j \in [N]$, $\card{S_i}=\card{S_j} = k$ and $\card{S_i \cap S_j} \le \gamma$.
\end{theorem}

With the above technical lemmas, we are ready to present the proof of \Cref{th general low}.

\begin{proof}(Proof of \Cref{th general low})
    Let $\X$ be a space of $n=4k^2 \ln 4k$ examples. By Sauer's lemma, we know that any query family $Q$ with VC dimension $d$ must have
    \begin{align*}
    \card{\dom(Q)} \le \sum_{i=0}^d\binom{n}{i} \le cn^d < n^{d+1} = \left(4k^2 \ln 4k\right)^{d+1},  
    \end{align*}
     when $n$ is larger than some suitably large constant $c$. By \Cref{th intersection}, there exists some integer $N \ge (2k \ln 2k)^{\gamma+1}>(4k^2 \ln 4k)^{d+1}>\card{\dom(Q)}$, such that we are able to find subsets $C_1,\dots,C_N \subseteq X$, where each subset has size $k$ and any pair of these $N$ sets has at most $\gamma$ common examples. Notice that when $k$ is larger than some suitably large constant, $\gamma=3d$ is sufficient to make $(2k \ln 2k)^{\gamma+1}>(4k^2 \ln 4k)^{d+1}$.

    Our next step is to use the set family $\{C_1,\dots, C_N\}$ to construct our hypothesis class $\H$. For $i \in [N]$, define $\H_{C_i}=\{h_S \mid S \subseteq C_i, \card{S} \le 1\}$, where $h_S(x) = 1$ if and only if $x \in C^* \setminus S$. The hypothesis class we use is $\H=\bigcup_{i \in [N]}\H_{C_i}$. 
    
    We start by showing $\H$ has VC dimension at most $\gamma+2 \le 3d+2$. Let $I=\{x_1,\dots,x_{\gamma+3}\}$ be $\gamma+3$ different examples in $\X$. Assuming that $\H$ can shatter $I$, then we obtain that there exists some $h \in \H$ that labels every example in $I$ positive. By construction, there must be some $i \in [N]$ such that $h \in \H_{C_i}$, which implies that $I \subseteq C_i$. However, we next show that there is no $h' \in \H$ can label $x_1,\dots,x_{\gamma+1}$ positive but $x_{\gamma+2},x_{\gamma+3}$ negative. Assuming such an $h'$ exists, then there must be some $j \in [N]$ such that $\{x_1,\dots,x_{\gamma+1}\} \subseteq C_j$. However, if $j \neq i$, then $\card{C_i \cap C_j} \le \gamma$. Thus, we must have $i=j$, which means $h' \in \H_{C_i}$. However, by construction each hypothesis in $\H_{C_i}$ can only label at most one example contained in $C_i$ negative, which gives us a contradiction. So, we conclude the hypothesis class $\H$ we use has VC dimension at most $\gamma+2=3d+2$.

    Next, we show that assuming the labeling domain, the dataset and the space of examples are the same. i.e. $L=S=\X$, for every query class $Q$ with VC dimension at most $d$ there exists a subset of $k$ examples $C^*$ such that every learner $\A$ that makes less than $k/(2\gamma)$ queries will with probability at least $1/3$ mislabel some example $x \in C^*$. Since $N>\card{\dom(Q)}$, by \Cref{lm inclusion}, we know that there exists some $C^* \in \{C_1,\dots, C_N\}$ such that for every query $(T,z) \in Q$, if $T \subseteq C^*$, then $\card{T} \le \gamma$. By \Cref{lm cover}, we know that if $\A$ only makes less than $k/(2\gamma)$ queries from $Q$ then with probability at least $1/3$ some example $x \in C^*$ will be mislabeled by $\A$.

    Thus, for every $d$ and every $k$ that is larger than some constant, we constructed a hypothesis class $\H$ with VC dimension at most $3d+2$ over an example space $\X$ with size $\Tilde{O}(k^2)$, which is also the dataset $S$ to be labeled, such that for every learner $\A$ and query class $Q$ with VC dimension at most $d$, if $\A$ makes less than $k/2d$ queries than there is a true hypothesis $h^* \in \H$ such that with probability at least $1/3$, $\A$ will misclassify at least one of the examples, even assuming the labeling domain $L=S$.
 \end{proof}

We remark that the construction of the hard instance in \Cref{th general low} is fully combinatorial. So, given any large enough space of examples $\X$, we can embed the hard instance we constructed into $\X$ to get a hard instance in that example space.

\subsection{Proof of \Cref{col general low}}\label{app col}

\begin{corollary}[Restatement of \Cref{col general low}]
There is a space of examples $\X$ such that for every $d \in N^+$ and $n \ge d$ large enough, there exists a hypothesis class $\H$ over $\X$ with VC dimension $d$ such that there exists a set of $n$ example $S$ such that
for every region query family $Q$ over $\X$ with $VC\dim(Q) \le (d-3)/3$ and every active learning algorithm $\A$, there exists a true hypothesis $h^* \in \H$, such that if $\A$ makes less than $\poly(n)$ region queries from $Q$, then with probability at least $1/3$, some example $x \in S$ is labeled incorrectly by $\A$. In particular, this even holds when $\A$ knows the labeling domain $L=S$.
\end{corollary}

\begin{proof}(Proof of \Cref{col general low})
 For each $m$, let $\X_m=\{x^{(m)}_i\}_{i=1}^m$ be the space of examples constructed in \Cref{th general low} with parameter $m$. Let $\X=\cup_{m}\X_m$ be a space of examples. Since the constructions of $\X_m$ are fully combinatorial, we can assume for each $m_1,m_2 \in N^+$, $\X_{m_1} \cap \X_{m_2} = \emptyset$.
 
 Let $d \in N^+$ and let $H_m$ be the hypothesis class over $\X_m$ with VC dimension $d$ constructed in \Cref{th general low}. 
 For each $m$ and for each $f \in H_m$, we extend $f$ to $\X$ in the following way. For every $x \in \X \setminus \X_m$, $f(x)=-1$. $H_m$ still has VC dimension $d$ over $\X$ under the extension. Furthermore, since each $\X_m$ is disjoint, $H= \cup_m H_m$ has VC dimension at most $d+1$. For each $n>d$ larger enough, let $S_n = \X_n \subseteq \X$ be a subset of $n$ examples. By \Cref{th general low}, we know that for every learning algorithm $\A$ and for every query class $Q$ with VC dimension at most $(d-2)/3$ there exists a hypothesis $h^* \in H_n \subseteq \H$ such that $\A$ must make $\poly(n)$ queries from $Q$ to perfectly label $S_n$ with probability more than $2/3$, even if $\A$ knows that a query will be checked based on $S_n$. 
\end{proof}

\section{Missing Details in \Cref{sec efficient}}\label{app efficient}

In this section, we design efficient learning algorithms for several concrete hypothesis classes including the class of union of $k$ intervals, the class of high dimensional boxes, and the class of high dimensional halfspaces, giving missing details in \Cref{sec efficient}.

\subsection{Proof of \Cref{th k-subset}}\label{app interval}
In this section, we prove \Cref{th k-subset} by designing an efficient learning algorithm for the class of the union of $k$ intervals. We restate \Cref{th k-subset} as follows.

\begin{theorem}(Restatement of \Cref{th k-subset})
    Let $\X=\R$ be the space of examples and $\H=\{h \mid \exists [a_i,b_i], i \in [k], s.t. h(x)=1 \iff x \in \cup_{i=1}^k [a_i,b_i] \}$ be the class of union of $k$ intervals over $\R$. Let $I$ be the class of intervals over $\R$ and query family $Q=\{(T,z) \mid T \in I, z \in \{\pm 1\}\}$. There is a learner $\A$ such that for every subset of $n$ examples $S$, labeled by any $h^* \in \H$ and for every labeling domain $S \subseteq L$(possibly unknown to $\A$), $\A$ runs in $O((T+n)k\log n)$ time, makes $O(k\log n)$ queries from $Q$ and labels every example in $S$ correctly, where $T$ is the running time to implement a single region query.
\end{theorem}

We start with \Cref{alg findleft}, a sub-routine used to label examples in the left-most interval of the target hypothesis. The guarantee of \Cref{alg findleft} is presented in \Cref{lm left}.

\begin{algorithm}[ht]
		\caption{\textsc{FindLeft}$(S)$ (Find the smallest $i^*$ such that $q([x^{(1)},
x^{(i^*)}],y)=1$) for some $y \in \{\pm 1\}$}\label{alg findleft}
		\begin{algorithmic} 
\State Order $S$ as $x^{(1)}<\dots<x^{(m)}$, where $m=\card{S}$.
\If{$q([x^{(1)},x^{(m)}],y)=1$ for some $y \in \{\pm 1\}$}
\State \Return $m$
\EndIf
\State Let $C=\{x^{(1)},\dots,x^{(m)}\}$ \Comment{Candidates of the boundary points}
\While{$\card{C}>1$} \Comment{Find the boundary via binary search}
\State Let $x'$ be the median of $C$. \Comment{If $\card{C}$ is even, select $x'$ as the larger one}
\If{$q([x^{(1)},x'],y)=0$, $ \forall y \in \{\pm 1\}$} 
\State Remove $x'$ and all points greater than $x'$ from $C$
\Else \State Remove all points less than $x'$ from $C$
\EndIf
\EndWhile
\Return the index of the single element in $C$ 
		\end{algorithmic}
	\end{algorithm}

\begin{lemma}\label{lm left}
    Let $S=(x^{(1)},\dots,x^{(m)}) \subseteq \R$ be a subset of $n$ examples labeled by a union of $k$ intervals $h^*=\cup_{i=1}^k [a_i,b_i]$. Let $L \subseteq \R$ be any arbitrary labeling domain such that $S \subseteq Y$. \textsc{FindLeft}$(S)$ makes $O(\log m)$ interval queries and returns the smallest index $i^*$ such that there is some $y \in \{\pm 1\}$ such that $q([x^{(1)},x^{(i^*)}],y)=1$ and for every $y \in \{ \pm 1\}$, $q([x^{(1)},x^{(i^*+1)}],y)=0$.
\end{lemma}

\begin{proof}(Proof of \Cref{lm left})
We first notice that if $q([x^{(1)},x^{(i)}],y)=0$, $\forall y \in \{ \pm 1\}$, then $\forall j>i$, we also have $q([x^{(1)},x^{(j)}],y)=0$, $\forall y \in \{ \pm 1\}$. This is because for any labeling domain $L$, $[x^{(1)},x^{(i)}] \cap L \subseteq [x^{(1)},x^{(j)}] \cap L$. Thus, if $[x^{(1)},x^{(i)}] \cap L$ contains examples with both positive examples and negative examples then so does $[x^{(1)},x^{(j)}] \cap L$. This implies that if some $x'$ such that $q([x^{(1)},x'],y)=1$, for some $y \in \{\pm 1\}$, is removed from $C$, then we must have found some $x''>x'$ such that $q([x^{(1)},x'],y)=1$, for some $y \in \{\pm 1\}$. In particular, this implies that $x^{(i^*)}$, where $i^*$ is the index that satisfies the statement, will never be removed from $C$. This proves the correctness of \Cref{alg findleft}. It remains to prove the query complexity of \Cref{alg findleft}. In each iteration of \Cref{alg findleft}, we only make at most $2$ region queries and remove half of the remaining points in $C$. This implies that \Cref{alg findleft} will run at most $O(\log(\card{C}))$ iterations and the query complexity is $O(\log m)$.
\end{proof}

Given \Cref{alg findleft} and \Cref{lm left}, we are now ready to present \Cref{alg ksubset}, the learning algorithm and the proof of \Cref{th k-subset}.

\begin{algorithm}[ht]
		\caption{\textsc{Label $k$-Interval}$(S,\H)$ (Label $S$ with interval queries given hypothesis class $\H$ )}\label{alg ksubset}
		\begin{algorithmic}

\While {$\card{S}>0$}
\State Order $S$ as $x^{(1)}<\dots<x^{(m)}$, where $m=\card{S}$.
\State $i^* \gets \textsc{FindLeft}(S)$ 
\If{$q([x^{(1)},x^{(i^*)}],1)=1$}
\State Label $x^{(1)},\dots,x^{(i^*)}$ by 1
\Else \State Label $x^{(1)},\dots,x^{(i^*)}$ by -1
\EndIf
\State $S \gets S \setminus \{x^{(1)},\dots,x^{(i^*)}\}$.
\EndWhile
		\end{algorithmic}
	\end{algorithm}

\begin{proof}(Proof of \Cref{th k-subset})
    We first show the correctness of \Cref{alg ksubset}. By \Cref{lm left}, we know that every time \Cref{alg ksubset} calls \Cref{alg findleft}, we will find some $i^*$ such that for some $y \in \{\pm 1\}$, $q([x^{(1)},x^{(i^*)}],y)=1$, which implies that the true labels of $x^{(1)},\dots, x^{(i^*)}$ are $y$. Thus, \Cref{alg ksubset} labels every example correctly.

Next, we bound the query complexity of \Cref{alg ksubset}. By \Cref{lm left}, we know that each time we call \Cref{alg findleft}, the example $x^{(i^*)}$ satisfies the following property. There is some $y \in \{\pm 1\}$ such that $q([x^{(1)},x^{(i^*)}],y)=1$ but $q([x^{(1)},x^{(i^*+1)}],y)=0$. 
Since the target hypothesis $h^*$ is a union of $k$ intervals, over any dataset $S$, there are at most $2k$ such pair of $x^{(i^*)}$ and $x^{(i^*+1)}$. Each call of \Cref{alg findleft} finds one of such pair. Thus \Cref{alg ksubset} calls \Cref{alg findleft} at most $2k$ times. By \Cref{lm left}, each call of \Cref{alg findleft} will make $O(\log n)$ queries. Thus, the query complexity of \Cref{alg ksubset} is $O(k \log n)$.

Furthermore, we notice that the running time of \Cref{alg findleft} is $O((T+n)\log n)$, since each time we do a binary search, make 2 region queries and remove examples from the candidate set $C$, which takes $O(T+n)$ time. Thus the running time of \Cref{alg ksubset} is $O((T+n)k\log n)$.
\end{proof}

\subsection{Proof of \Cref{th box}}\label{app box}

We present the proof of \Cref{th box}, restated as follows.

\begin{theorem}(Restatement of \Cref{th box})
    Let $\X=\R^d$ be the space of examples and $\H=\{\prod_{i=1}^d [a_i,b_i] \mid a_i,b_i \in [-\infty,\infty]\}$ be the class of axis-parallel boxes in $\R^d$ that labels $\X$. There is a query class $Q$ over $\R^d$ with VC dimension $O(\log d)$ and an efficient algorithm $\A$ such that for every set of $n$ examples $S \subseteq \R^d$, every target hypothesis $h^* \in \H$, and for every labeling domain $S \subseteq L$(possibly unknown to $\A$), $\A$ runs in $O((T+n)d\log n)$ time, makes $O(d\log n)$ queries from $Q$ and labels every example in $S$ correctly, where $T$ is the running time to implement a single region query.
\end{theorem}

Similar to what we did in Appendix~\ref{app interval}, we start with \Cref{alg find boundary}, a subroutine we use to approximately learn a boundary of the target hypothesis. The guarantee of \Cref{alg find boundary} is presented in \Cref{lm boundary}.

\begin{algorithm}[ht]
		\caption{\textsc{FindBoundary}$(S,w)$ (Find the boundary point in $S$ along direction $w$)}\label{alg find boundary}
		\begin{algorithmic} 
\State Order $S$ as $x^{(1)},\dots, x^{(m)}$, where $m=\card{S}$, such that $w\cdot x^{(1)} \le \dots \le w\cdot x^{(m)}$.
\If{$q(\{x \mid w \cdot x \ge w \cdot x^{(1)}\},0)=1$}
\State \Return $-\infty$
\EndIf
\State Let $C=\{w \cdot x^{(1)},\dots,w\cdot x^{(m)}\}$ \Comment{Candidates of boundary points}
\While{\card{$C$}>1} \Comment{Find the boundary point via binary search}
\State Let $b$ be the median of $C$. \Comment{If $\card{C}$ is even, select $b$ as the larger one}
\If{$q(\{x \mid w \cdot x \ge b\},-1)=1$} 
\State Remove $b$ and all elements greater than $b$ from $C$
\Else \State Remove all elements less than $b$ from $C$
\EndIf
\EndWhile
\State
\Return the single element in $C$ 
		\end{algorithmic}
	\end{algorithm}

To prove \Cref{th box}, we first prove the following technical lemma.
\begin{lemma}\label{lm boundary}
    Let $S \subseteq \R^d$ be a subset of $n$ examples labeled by an axis-parallel box $h^*=\prod_{i=1}^d [a_i,b_i]$. Let $L \subseteq \R^d$ be any arbitrary labeling domain such that $S \subseteq L$.
    For every $i \in [d]$, $\textsc{FindBoundary}(S,e_i)$ returns $\hat{b}_i\le b_i$ by making $O(\log n)$  queris such that for every $x \in S$ with $x_i>\hat{b}_i$, $x$ is labeled by $-1$. Symmetrically, $\textsc{FindBoundary}(S,-e_i)$ returns $\hat{a}_i\ge a_i$ by making $O(\log n)$ queris such that for every $x \in S$ with $x_i<\hat{a}_i$, $x$ is labeled by $-1$.   
\end{lemma}

\begin{proof}(Proof of \Cref{lm boundary})
    We prove the case for $\textsc{FindBoundary}(S,e_i)$ and the case for $\textsc{FindBoundary}(S,-e_i)$ can be proved symmetrically. 
    We first prove the correctness of the algorithm.
    If \Cref{alg find boundary} terminates in the first round ($q(\{x \mid e_i \cdot x \ge e_i \cdot x^{(1)}\},-1)=1$), then clearly $-\infty=\hat{b}_i \le b_i$. Furthermore, since $S \subseteq \{x \mid e_i \cdot x \ge e_i \cdot x^{(1)}\} \cap L $, we know that every example in $S$ is labeled by $-1$. In this case, the statement of \Cref{lm boundary} is true. In the rest of the proof, we assume \Cref{alg find boundary} does not terminate in the first round.
    We observe that for every $1 \le i<j \le m$, we have $  \{x \mid e_i \cdot x \ge e_i \cdot x^{(j)}\} \subseteq\{x \mid e_i \cdot x \ge e_i \cdot x^{(i)}\}$. This implies that there exists a largest index $j^* \in [m]$ such that $q(\{x \mid e_i \cdot x \ge e_i \cdot x^{(j^*)}\},-1)=0$. In particular, $x^{(j^*)}_i \le b_i$, because otherwise, any example $x \in \R^d$ with $x_i \ge x^{(j^*)}_i$ will be labeled by $-1$, which gives a contradiction to the answer to $q(\{x \mid e_i \cdot x \ge e_i \cdot x^{(j^*)}\},-1)$. So, it is sufficient to show that the output $\hat{b}_i$ of \Cref{alg find boundary} is $x^{(j^*)}_i$.

    Assuming we receive a feedback $q(\{x \mid e_i \cdot x \ge e_i \cdot x^{(k)}\},-1)=1$ for some $x^{(k)} \in S$, then no $x^{(j)}_i$ with $j<k$ is removed from $C$. In particular, no $x^{(j)}_i$ with $q(\{x \mid e_i \cdot x \ge e_i \cdot x^{(j)}\},-1)=0$ is removed from $C$. This implies that the final element remained in $C$ must be some $x^{(j)}_i$ with $q(\{x \mid e_i \cdot x \ge e_i \cdot x^{(j)}\},-1)=0$. On the other hand, suppose we are removing some $x^{(k)}_i$ with $q(\{x \mid e_i \cdot x \ge e_i \cdot x^{(k)}\},-1)=0$. This implies we received a feedback of the form $q(\{x \mid e_i \cdot x \ge e_i \cdot x^{(k')}\},-1)=0$ for some $k'>k$. Thus, any $x^{(j)}_i$ with $j<k$ is either removed together with $x^{(k)}_i$ or has already been removed from $C$. This implies that the single element remaining in $C$ is $x^{(j^*)}_i$, which is the output.

    Finally, it remains to show the query complexity of \Cref{alg find boundary} is $O(\log n)$. Since $b$ is selected as as the median of $C$, after every query, we remove half elements from $C$. So after making at most $O(\log n)$ queries, there is a single element remained in $C$ and is output by \Cref{alg find boundary}.
\end{proof}

Given \Cref{alg find boundary} and \Cref{lm boundary}, we are ready to present the learning algorithm, \Cref{alg box} and the proof \Cref{th box}.

\begin{algorithm}[ht]
		\caption{\textsc{LabelBox}$(S,\H)$ (Label $S$ with halfspace query  given hypothesis class $\H$ )}\label{alg box}
		\begin{algorithmic} 
\For{$i \in [d]$}
\State $x^i_r \gets \textsc{FindBoundary}(S,e_i)$ 
\State $x^i_l \gets \textsc{FindBoundary}(S,-e_i)$
\EndFor
\State Label all examples in $S \cap \prod_{i=1}^d[x^i_l,x^i_r]$ to be 1 and the others to be $-1$.
		\end{algorithmic}
	\end{algorithm}

\begin{proof}(Proof of \Cref{th box})
We first prove the correctness of \Cref{alg box}.
   Let $B=\prod_{i=1}^d [a_i,b_i]$ be the target box that labels $S$. By the first part of \Cref{lm boundary}, we know that the estimator $\hat{B}=\prod_{i=1}^d[x^i_l,x^i_r]$ of \Cref{alg box} is a subset of $B$. Thus, every negative example in $S$ is also labeled negative by \Cref{alg box}. Furthermore, by the second part of \Cref{lm boundary}, we know that any example $x \in S \setminus \hat{B}$ has a true label $-1$. Thus, no positive example in $S$ is labeled incorrectly.
   
Next, we upper bound the query complexity of \Cref{alg box}. By \Cref{lm boundary}, we know that every time we call \Cref{alg find boundary}, we make $O(\log n)$ queries. So the query complexity of \Cref{alg box} is $O(d\log n)$. Furthermore, since every time we call \Cref{alg find boundary} in \Cref{alg box}, we just do a binary search. The running time of \Cref{alg box} is $O((T+n)d\log n)$.

Finally, we show the query class $Q$ has a small VC dimension. Notice that each query in $Q$ corresponds to some linear classifier $\{x \mid w \cdot x \ge b\}$, where $w$ is parallel to some $e_i$ for $i \in [d]$. According to \cite{gey2018vapnik}, we know that $Q$ has VC dimension $O(\log d)$.
\end{proof}

\subsection{Learning Arbitrary High-Dimensional Halfspaces}\label{app halfspace}

In this section, we move to our main algorithmic result for learning halfspace. Since in this work, we want to label an arbitrary dataset $S \subseteq \R^d$, we can without loss of generality to assume that the target halfspace is homogeneous $w^*$.

\subsubsection{Efficient Halfspace Learning with Arbitrarily Complicated Query Family}\label{app self}
As we discussed in \Cref{sec halfspace}, we will first give an efficient halfspace learning algorithm using an arbitrarily complicated query class using the connection between active learning with region queries and self-directed learning.
To start with, we remind the readers of the model of self-directed learning.



\begin{definition}[Self-Directed Learning\cite{goldman1994power}]
    \label{def:self-directed-learning}
Let $\X$ be a space of examples and let $\H$ be a class of hypothesis over $\X$.
Let $h^* \in \H$ be an unknown target hypothesis let $S =
\{x^{(1)},\ldots, x^{(n)}\} \subseteq \X$ be a subset of $n \in \mathbb N$ examples.
The learner has access to the full set of (unlabeled) points $\X$. 
\\
Until the labels of all examples of $S$ have been predicted: 
\vspace{-0.3em}
\begin{itemize} 
\itemsep0em
    \item The learner $\A$ picks a point $x \in S $ and makes a prediction $\hat{y} \in \{0,1\}$ about its label.
    \item The true label $h^*(x)$ of $x$ is revealed and the learner makes a mistake if $\hat{y} \neq h^*(x)$
\end{itemize}
\vspace{-0.3em}
The mistake bound $M(\A,S,h^*)$ is the total number of mistakes that $\A$ makes during the learning process.
\end{definition}

\begin{theorem} \label{th online active}
    Let $\X$ be a space of example and $\H$ be a class of hypotheses over $\X$. Let $S \subseteq X$ be a subset of $n$ examples and let $h^* \in \H$ be the target hypothesis. Let $Q=\{(T,z) \mid z \in \{\pm 1\},T\subseteq 2^{\X}\}$ over $\X$ be the query class that contains any subset of $\X$.
    If there is a self-directed learner $\A$ with mistake bound $M=M(\A,S,h^*)$ that labels $S$, then there is an active learner $\A'$ that makes $O(M\log n)$  queries from $Q$ and labels $S$ correctly in time $M(nT_\A+T_Q\log n)$, where $T_A$ is the running time of $\A$ to predict a single example and $T_Q$ is the running time to implement a single query. 
\end{theorem}

\begin{proof}(Proof of \Cref{th online active})
    We construct $\A'$ as follows. In a single round, if there is still an example $x \in S$, for which we don't know the true label, we run the self-directed learner $\A$ over $S$ from the beginning to predict every example in $S$. If $\A$ makes a prediction for some $x$, whose label is already known, we provide the true label for $\A$ as feedback, otherwise, we provide the prediction of $\A$ as feedback assuming the prediction of $\A$ is correct. Denote by $\{(x,\Tilde{y}(x))\}_{x \in S}$ the feedback that $\A$ receives in this execution. Denote by $S^+:=\{x \in S \mid \Tilde{y}(x)=1\}$ and $S^-:=\{x \in S \mid \Tilde{y}(x)=-1\}$. We make two queries $(S^+,1)$ and $(S^-,-1)$. If $q(S^+,1)=q(S^-,-1)=1$, we label every $x \in S$ according to $\Tilde{y}(x)$. Otherwise, we do a binary search over $S^+$ and $S^-$ to find the first example $x'$ where $\A$ makes a mistake in this execution. Then we know the true label of every example up to $x'$, and we enter the next round. Clearly, when $\A'$ terminates, we label every example in $S$ correctly.

    Next, we upper bound the query complexity of $\A'$.
    We notice that in every round of execution of $\A'$, before the self-directed learner $\A$ predicts some example $x$ whose label we don't know and $\A$ actually makes a mistake at $x$, the performance of $\A$ in this setting is the same as the performance of $\A$ who receives the true feedback. When $\A$ actually makes a mistake at $x$, we use $O(\log n)$ region queries to do a binary search and find the first misclassified examples $x$ whose label we don't actually know. This implies in the next round $\A$ will be fed with the true feedback at example $x$ and $\A$ will keep performing well until it makes the next mistake at some example we don't know the true label. Since the mistake bound of $\A$ is $M$, we know that $\A'$ will have at most $M$ rounds and thus the total query complexity is $O(M\log n)$.

    Finally, we analyze the running time of $\A'$. As we analyzed in the last paragraph, we know $\A'$ in total have at most $M$ rounds, in each round, we make $n$ predictions and make $O(\log n)$ queries. Thus, the running time of $\A'$ is $O(M(nT_\A+T_Q\log n))$.
\end{proof}

Given \Cref{th online active}, to prove \Cref{th SVM}, it is sufficient to design an efficient self-directed halfspace learning algorithm that makes $O(d\log n)$ mistakes for every $S$ any every target halfspace $w^*$. Such an algorithm is easy to design using the idea from \cite{haussler1994predicting}.

\begin{theorem}\label{th random}
    Let $\X=\R^d$ be the space of the examples and let $\H = \{w \mid w \in S^{d-1}\}$ be the class of homogeneous halfspaces in $\R^d$ that labels $\X$. There is an efficient self-directed learning algorithm $\A$ such that for every subset $S\subseteq \X$ of $n$ examples and for every target hypothesis $w^* \in \H$, $\A$ predicts each example in time $\poly(B)$, where $B$ is the bit complexity of $S$ and makes $O(d\log n)$ mistakes in expectation
\end{theorem}

\begin{proof}(Proof of \Cref{th random})
    We first describe the self-directed learning algorithm. The algorithm randomly order $S$ and obtain a sequence of example $x^{(1)},\dots, x^{(n)}$. To predict the label of example $x^{(i+1)}$, it computes $w^{(i)}$, a solution of the support vector machine (SVM) of $(x^{(1)},y^{(1)}), \dots, (x^{(i)},y^{(i)})$ and predicts $\hat{y}^{(i+1)} = \sgn(w^{(i)} \cdot x^{(i+1)})$. 

    Now denote by $w^{(i+1)}$ the solution of SVM of $(x^{(1)},y^{(1)}), \dots, (x^{(i+1)},y^{(i+1)})$. Notice that $w^{(i+1)}$ is uniquely determined by the $d$ support vectors. Since we make a random permutation of $S$, the probability that $x^{(i+1)}$ is one of the support vector is at most $d/(i+1)$, which implies that with probability at most $d/(i+1)$, $w^{(i)} \neq w^{(i+1)}$. Thus, the probability that we make a mistake at $x^{(i+1)}$ is at most $d/{(i+1)}$. Denote by $M$, the total number of mistakes made by $\A$. We have 
    \begin{align*}
        \E M = \sum_{i=1}^n \E \mathbbm 1 (x_i \text{ is misclassified}) = \sum_{i=1}^n \Pr\mathbbm 1 (x_i \text{ is misclassified}) \le \sum_{i=1}^n \frac{d}{i} \le O(d\log n).
    \end{align*}
This shows in expectation the mistake bound of $\A$ is $O(d\log n)$.    
Furthermore, every time $\A$ makes a prediction, it solves a convex program based on $S$, and thus the running time is $\poly(B)$.
\end{proof}

With \Cref{th online active} and \Cref{th random}, we give the following active learning algorithm and the proof of \Cref{th SVM}.

\begin{algorithm}
		\caption{\textsc{RandomizedSVM}$(S)$ (Label $S$ with arbitrary region query )}\label{alg SVM}
		\begin{algorithmic} 
\State Randomly order dataset $S$ and obtain sequence of examples $x^{(1)},\dots,x^{(n)}$
\State $i^* \gets 0$
\While{$i^*<n$}
\State Let $\hat{w}$ be a solution of the SVM over labeled data $(x^{(1)},y^{(1)}),\dots (x^{(i^*)},y^{(i^*)})$
\For{$i \in [n]$}
\State $\hat{y}^{(i)} \gets \sgn(\hat{w} \cdot x^{(i)})$
\If{$i>i^*$}

\State Update $\hat{w}$ to be a solution of the SVM over labeled data $(x^{(1)},\hat{y}^{(1)}),\dots (x^{(i)},\hat{y}^{(i)})$
\EndIf
\EndFor
\State Let $S^+:\{x^{(i)} \mid \hat{y}^{(i)} =1 \}$ and $S^-:\{x^{(i)} \mid \hat{y}^{(i)} =-1\}$
\State Make queries $(S^+,1)$ and $(S^-,-1)$.
\If{$q(S^+,1)=q(S^-,-1)=1$}
\State Label every $x^{(i)}$ by $\hat{y}^{(i)}$ and return
\Else
\State Binary search over $S^+$ and $S^-$ to find the smallest $j$ such that $\hat{y}^{(j)} \neq y^{(j)}$ via region queries.
\EndIf
\State $i^* \gets j$, $\hat{y}^{(i^*)} \gets 1-\hat{y}^{(i^*)}$ 
\State Label every $x^{(i)}$ by $\hat{y}^{(i)}$
\EndWhile
\State Label all examples in $S \cap \prod_{i=1}^d[x^i_l,x^i_r]$ to be 1 and the others to be $-1$.
		\end{algorithmic}
	\end{algorithm}

 \begin{theorem}(Restatement of \Cref{th SVM})
Let $\X = \R^d$ be the space of examples and $\H = \{w \mid w \in S^{d-1}\}$ be the class of homogeneous halfspaces in $\R^d$ that labels $\X$. Let $Q=\{(T,z) \mid z \in \{\pm 1\},T\subseteq 2^{\R^d}\}$ over $\R^d$ be the query class that contains any subset of $\R^d$. There is an efficient algorithm $\A$ such that for every set of $n$ examples $S$, labeled by any $w^* \in \H$ and for every labeling domain $S \subseteq L$ (possibly unknown to $\A$), $\A$ runs in $O((T+B)d\log^2 n)$ time, makes $O(d\log^2 n)$ queries from $Q$ in expectation and labels every example in $S$ correctly, where $T$ is the running time to implement a single query and $B$ is the bit complexity of $S$.
\end{theorem}

\begin{proof}(Proof of \Cref{th SVM})
The proof of \Cref{th SVM} follows directly by \Cref{th online active} and \Cref{th random}. The algorithm we use is \Cref{alg SVM}, which converts the self-directed learning algorithm used in the proof of \Cref{th random} to an active learner using the proof of \Cref{th online active}.
\end{proof}

\subsubsection{Efficient Halfspace Learning with Simple Query Family}\label{app main}
In this section, we design an efficient halfspace learning algorithm with low query complexity using a query class with $\poly(d)$-VC dimension and prove \Cref{th halfspace}.

\begin{theorem}(Restatement of \Cref{th halfspace})
    Let $\X=\R^d$ be a space of examples and let $\H = \{w \mid w \in S^{d-1}\}$ be the class of homogeneous halfspaces in $\R^d$ that labels $\X$.
    There is an algorithm and a query class $Q$ with VC dimension $\Tilde{O}(d^3)$ such that for every subset of $n$ examples $S \subseteq \X$, every labeling domain $L$ with $S \subseteq L$ and every target hypothesis $w^* \in \H$ and every $\alpha \in (0,1)$, it in expectation makes $O(d^3\log^2d \log (1/\alpha))$  queries from $Q$, runs in $\poly(d,n,T)$ time and labels $(1-\alpha)$ fraction of examples in $S$ correctly, where $T$ is the running time of implement a single region query from $Q$. 
    In particular, the algorithm makes $O(d^3\log^2d\log n)$ queries from $Q$ and labels every example in $S$ correctly in time $\poly(d,n,T)$.
\end{theorem}

As mentioned in \Cref{sec halfspace}, we will make use of Forster's transform to make our dataset well-behaved. So, we will start with some background on Forster's transform before diving into the proof. We first introduce the notion of Approximate Radially Isotropic Position.



\begin{definition}[Approximate Radially Isotropic Position]
    Let $S$ be a multiset of non-zero points in $\R^d$, we say $S$ is in $\epsilon$-approximate radially isotropic position, if for every $x \in S$, $\norm{x}=1$ and for every $u \in S^{d-1}$, $\sum_{x \in S}(u\cdot x)^2/\card{S} \ge 1/d-\epsilon$.
\end{definition}

A simple calculation gives the following useful result, which has appeared in \cite{diakonikolas2023strongly,diakonikolas2023self}, for a dataset that is in an approximate radially isotropic position.
\begin{lemma}\label{lm isotropic}
    Let $S$ be a multiset of non-zero points in $\R^d$ that is in $1/2d$-approximate radially isotropic position. Then for every $u \in S^{d-1}$, we have $\Pr_{x \sim S} \left(\abs{u\cdot x} \ge 1/2\sqrt{d}\right) \ge 1/4d$.
\end{lemma}
In particular, several works have been done to show an approximate Forster's transform can be computed efficiently.

\begin{theorem}[Approximate Forster's Transform \cite{diakonikolas2023strongly}]\label{th forster}
    There is an algorithm such that given any set of $n$ points $S \subseteq \R^d \setminus \{0\}$ and $\epsilon>0$, it runs in time $\poly(d,n, \log 1/\epsilon)$ and returns a subspace $V$ of $\R^d$ containing at least $\dim(V)/d$ fraction of points in $S$ and an invertible matrix $A \in \R^{d \times d}$ such that $f_A(S \cap V)$ is in $\epsilon$-approximate radially isotropic position up to isomorphic to $\R^{\dim(V)}$, where $f_A(S \cap V) = \{f_A(x):=Ax/\norm{Ax} \mid x \in S \cap V\}$.
\end{theorem}

Combine \Cref{th forster} and \Cref{lm isotropic}, we know that given any set of $n$ examples $S \subseteq \R^d$, we can find a subset of at least $kn/d$ examples $S_V:=S \cap V \subseteq S$ that lies in some $k$-dimensional subspace $V$ and some invertible matrix $A$ such that $f_A(S_V)$ is in $1/2k$-approximate radially isotropic position (up to isomorphic to $\R^{k}$). Now, for convenience, we assume our transformed data $f_A(S_V)$ is exactly our original dataset and we focus on the transformed data. Notice that for each $x \in S_V$, we have 
\begin{align*}
    \sgn(w^* \cdot x) =  \sgn(A^{-T}w^*  \cdot A x) = \sgn(A^{-T}w^*  \cdot f_A(x)) = \sgn( \proj_{A(V)}(A^{-T}w^*)  \cdot f_A(x)),  
\end{align*}
which implies that each transformed example $f_A(x)$ is labeled by halfspace $w^*_V$ and has the same label as $x$,  where $v^*$ is the unit vector parallel to $\proj_{A(V)}(A^{-T}w^*)$. (We can without loss of generality assume that $\proj_{A(V)}(A^{-T}w^*) \neq 0$, otherwise we only need to use a single query to check if examples in $V$ are all labeled positive.) Given the above discussion, we design \Cref{alg perceptron}, a learning algorithm that correctly labels a large fraction of the dataset $S$, if $S$ is in approximate radially isotropic position. Formally, we prove \Cref{th perceptron}.

\begin{theorem}\label{th perceptron}
Let $\X=\R^d$ be a space of examples and let $\H = \{w \mid w \in S^{d-1}\}$ be the class of homogeneous halfspaces in $\R^d$ that labels $\X$. 
    Let $S \subseteq \R^d$ be a set of $n$ examples that are classified by some unknown halfspace $w^* \in S^{d-1}$. Let $w \in S^{d-1}$ be a unit vector such that $w\cdot w^* \ge \Omega(1/\sqrt{d})$. Let $L$ be any labeling domain such that $S \subseteq L$.
    Denote by $L$ the output of \Cref{alg perceptron} with input $(w,S)$.
    If $S \subseteq S^{d-1}$ is in $1/2d$-approximate radially isotropic position, then \Cref{alg perceptron} makes $O(d^2\log^2 d)$ queries from a query family $Q$ with VC dimension $O(d^2)$, runs in $\poly(d,n,T)$ time and returns $L$ such that each $(x,y) \in L$, $y=w^*(x)$ and $\card{L} \ge n/4d$. Here, $T$ is the running time to implement a single query from $Q$.
\end{theorem}

\begin{proof}(Proof of \Cref{th perceptron})
    We first show the correctness of the algorithm. i.e. Each element in the labeled set $L$ has the correct label. No matter what the labeling domain is if some query $q(Z,y)=1$, then every example $x \in Z \cap S$ must has a label $y$. Thus, each labeled example in the output is correctly labeled.

    In the second step, we show that when $w$ and $w^*$ have a good correlation, \Cref{alg perceptron} will terminate after $\Omega(d\log d)$ iterations. In particular, we show the following robust proposition of the modified perceptron algorithm. We have the following claim.

\begin{claim}\label{cl perceptron}
    Let $w^*,w_0 \in \R^d$ be two unit vectors such that $w^*\cdot w_0 \ge \Omega(1/\sqrt{d})$. Assume the following update, $w_{t+1} = w_t-x_t(x_t\cdot w_t)$ and for some $t_0 \in Z$, such that for every $t \le t_0$, $\abs{x_t\cdot w_t} \ge \norm{w_t}/2\sqrt{d}$ and $(x_t\cdot w_t)(x_t\cdot w^*) \le 1/\poly(d)$. Then we have $t_0 \le O(d\log d)$.
\end{claim}

We prove the claim here. By the update, we have 
\begin{align*}
    w_{t+1}\cdot w^* = (w_t-x_t(x_t\cdot w_t)) \cdot w^* = w_t\cdot w^* - (x_t\cdot w_t)(x_t\cdot w^*) \ge w_t \cdot w^* - \frac{1}{d^2} \ge w_0 \cdot w^* - \frac{t_0}{d^2}.
\end{align*}
On the other hand, we have 
\begin{align*}
    \norm{w_{t+1}}^2 \le \norm{w_t}^2 - (w_t \cdot x_t)^2 \le (1-1/2d) \norm{w_t}^2. 
\end{align*}
If $t_0 \ge \Omega(d\log d)$, then before reach $t_0$, at some point we will have $\frac{w_{t} \cdot w^*}{\norm{w_t}}>1,$ which gives a contradiction.

Now, we will show \Cref{alg perceptron} terminates before $t \ge \Omega(d\log d)$, by showing that in each iteration, the example $x_t$ we use to update $w_t$ satisfies the condition in the statement of Claim~\ref{cl perceptron}. There are two cases to consider. 

In the first case, we update $w_t$ via some $x_t=x' \in S \cap \{x \in B \mid yv_0 \cdot x \ge \frac{1}{2\sqrt{d}} \}$ because $(w_t \cdot x_t)(w^* \cdot x) \le 0$. Clearly, $x_t$ is an example that satisfies the update requirement. 

In the second case, we know that the example $x'$ is correctly labeled by our current hypothesis $w_t$. In this case, according to \Cref{alg perceptron}, we partition the region $\{x \in B \mid yv_0 \cdot x \ge \frac{1}{2\sqrt{d}} \}$ into boxes with diameter $1/\poly(d)$.
Since $q(\{x \in B \mid yv_0 \cdot x \ge \frac{1}{2\sqrt{d}} \},y)=0$, we know that there must be a ``point'' $x'' \in \{x \in B \mid yv_0 \cdot x \ge \frac{1}{2\sqrt{d}} \}$ such that $x''$ is labeled incorrectly by $w_t$.
Although such a point $x''$ may or may not be in our dataset $S$, it must be in one of these small boxes. Thus, \Cref{alg perceptron} will finally find such a small box such that
\begin{align*}
q(\{x \in B \mid yv_k \cdot x\in [a^{(k)},b^{(k)}], 0\le k \le d-1 \} \cup \{x'\},y)=0.    
\end {align*}
Since $x'$ has a label $y$, we know that no matter what the labeling domain is, there must be some $x''$ labeled incorrectly by $w_t$ that is in this small box. Let $\Tilde{x}$ be any point in the box. If $\Tilde{x}$ is actually labeled incorrectly by $w_t$, then we make a good enough update because $(\Tilde{x} \cdot w_t)(\Tilde{x} \cdot w^*) \le 0$. Otherwise, we show such an update is not that bad.
Since $\abs{a^{(k)}-b^{(k)}} \le 1/\poly(d)$, we know that $\norm{\Tilde{x}-x''} \le d/\poly(d) = 1/\poly(d)$. On the other hand, each update will increase the length of $w_t$ by at most $1$, which implies that $\norm{w_t} \le t+1 \le O(d\log d)$.
This implies that 
\begin{align*}
    (\Tilde{x}\cdot w_t)(\Tilde{x}\cdot w^*)= (\Tilde{x}\cdot w_t)(x''\cdot w^*+ (\Tilde{x}-x'') \cdot w^*) \le (\Tilde{x}\cdot w_t)((\Tilde{x}-x'') \cdot w^*) \le \frac{O(d\log d)}{\poly(d)} = \frac{1}{\poly(d)}.
\end{align*}
So in general, we have $(\Tilde{x}\cdot w_t)(\Tilde{x}\cdot w^*) \le 1/\poly(d)$. In the meantime, since 
\begin{align*}
    \Tilde{x} \in \{x \in B \mid yv_k \cdot x\in [a^{(k)},b^{(k)}], 0\le k \le d-1 \} \cup \{x'\} \subseteq \{x \in B \mid yv_0 \cdot x \ge \frac{1}{2\sqrt{d}} \},
\end{align*}
we have $\abs{\Tilde{x}\cdot w_t} \ge \norm{w_t}/2\sqrt{d}$. So in each round the example $x_t$($x'$ or $\Tilde{x}$) we use to update $w_t$ always satisfies the update requirement thus after at most $t=O(d\log d)$ update, $w_t$ correctly label every example $x$ in the unit ball $B$ with $\abs{\Tilde{x}\cdot w_t} \ge \norm{w_t}/2\sqrt{d}$. When we reach this stage, no matter what the labeling domain is, as long as some example $x \in \{x \in B \mid yv_0 \cdot x \ge \frac{1}{2\sqrt{d}} \} \cap S$, we always have 
\begin{align*}
    q(\{x \in B \mid yv_0 \cdot x \ge \frac{1}{2\sqrt{d}} \},y)=1,
\end{align*}
which will make \Cref{alg perceptron} terminate. Furthermore, if $S$ is $1/2d$-approximate radially isotropic position, by \Cref{lm isotropic}, we know that 
\begin{align*}
    \card{\{x \in B \mid \abs{v_0 \cdot x} \ge \frac{1}{2\sqrt{d}} \} \cap S} \ge n/4d,
\end{align*}
which implies that \Cref{alg perceptron} correctly label $1/4d$-fraction of the examples. 

In the third step, we upper bound the query complexity and the running time of \Cref{alg perceptron}. In each perceptron update for $w_t$, we make $d$ binary searches over $\poly(d)$ cells to find an example to update $w_t$. Each binary search makes $O(\log d)$  queries and in total we make $O(d\log d)$ queries to make one update. Since we make at most $O(d\log d)$ updates, we know the query complexity of \Cref{alg perceptron} is $O(d^2\log^2 d)$. 

Finally, we upper bound the VC dimension of the query family $Q$. Since each query \Cref{alg perceptron} is a set of $O(d)$ $d$-dimensional linear inequalities, we know that the VC dimension of $Q$ is $O(d^2)$.
\end{proof}

\begin{algorithm}[ht]
		\caption{\textsc{ActivePerceptron}$(w,S)$ (Label a large fraction of example in $S$)}\label{alg perceptron}
		\begin{algorithmic} 
\State $d \gets \dim(S)$ $t \gets 0$  $w_t \gets w$  
\While{$t \le O(d\log d)$}
\State Let $v_0=w_t/\norm{w_t}$ and $B$ be the unit sphere
\State Make query $q(\{x \in B \mid yv_0 \cdot x \ge \frac{1}{2\sqrt{d}} \},y)$, if $\{x \in B \mid yv_0 \cdot x \ge \frac{1}{2\sqrt{d}} \} \cap S \neq 0$.
\If{Every query made above returns $1$}
\State \Return $\{(x,y) \mid x \in S, y\in\{-1,1\}, yv_0 \cdot x \ge \frac{1}{2\sqrt{d}} \}$
\Else
\State Let $y \in \{-1,1\}$ such that $q(\{x \in B \mid yv_0 \cdot x \ge \frac{1}{2\sqrt{d}} \},y)=0$
\State Query $q(\{x'\},y)$, where $x' \in S \cap \{x \in B \mid yv_0 \cdot x \ge \frac{1}{2\sqrt{d}} \}$
\If{$q(\{x'\},y)=0$} 
\State $x_t=x'$, $w_t \gets w_t-(w_t \cdot x_t) x_t$, $t \gets t+1$
\Else \State Let $1/2\sqrt{d}=\theta_0 \le \theta_1 \le \dots \le \theta_\ell=1$ such that $\theta_{i}-\theta_{i-1}=1/\poly(d)$.
\State Find some $[a^{(0)},b^{(0)}]:=[\theta_{i-1},\theta_{i}]$ for some $i \in [\ell]$ such that $q(\{x \in B \mid yv_0 \cdot x\in [a^{(0)},b^{(0)}] \} \cup \{x'\},y)=0$. 

\Comment{ This can be done with $O(\log d)$ queries via binary search by making query of the form $q(\{x \in B \mid yv_0 \cdot x \ge  \theta_j \} \cup \{x'\},y)$ }.
\State Let $v_1,\dots,v_{d-1}$ be a standard basis of the subspace orthogonal to $w_t$.
\State Let $-1=\theta_0 \le \theta_1 \le \dots \le \theta_\ell=1$ such that $\theta_{i}-\theta_{i-1}=1/\poly(d)$.
\For{$j \in [d-1]$}
\State Find some $[a^{(j)},b^{(j)}]:=[\theta_{i-1},\theta_{i}]$ for some $i \in [\ell]$ such that $q(\{x \in B \mid yv_k \cdot x\in [a^{(k)},b^{(k)}], 0\le k \le j \} \cup \{x'\},y)=0$ via binary search.
\EndFor
\State Let $\Tilde{x}$ be any point in $\{x \in B \mid yv_k \cdot x\in [a^{(k)},b^{(k)}], 0\le k \le d-1\}$
\State $x_t=\Tilde{x}$, $w_t \gets w_t-(w_t \cdot x_t) x_t$, $t \gets t+1$
\EndIf
\EndIf
\EndWhile
\Return $\emptyset$ \Comment{If $w$ is not a good initialization, no example will be labeled.}

		\end{algorithmic}
	\end{algorithm}

Finally, we present \Cref{alg halfspace}, the halfspace learning algorithm and the proof of \Cref{th halfspace}.

\begin{algorithm}[ht]
		\caption{\textsc{LearningLTF}$(S,\alpha)$ (Label $1-\alpha$ fraction of $S$ with simple query )}\label{alg halfspace}
		\begin{algorithmic} 
\State $L \gets \emptyset$, $n\gets \card{S}$
\While{$\card{L}<(1-\alpha)n$}
\State Apply \Cref{th forster} to $S$ with $\epsilon=1/2d$ to obtain a matrix $A$ and a $k$-dimensional subspace $V$
\If{$q(V,1)=1$}
\State $L_V \gets \{(f_A(x),1) \mid x \in S \cap V\}$
\Else \State $L_V \gets \emptyset$
\EndIf
\While{$\card{L_V} < \card{S \cap V}/4k$}
\State Draw $w_0$ uniformly from the unit sphere in $A(V)$
\State  $L_V \gets \textsc{ActivePerceptron}(w_0,f_A(S \cap V))$

\Comment{If $w_0$ is not a good initialization, $L_V=\emptyset$.}

\Comment{To run $\textsc{ActivePerceptron}(w_0,f_A(S \cap V))$, we implement each query $(Z,y)$ by query $(\{x \in V \mid f_A(x) \in Z\},y)$.}
\EndWhile
\State Label every $x \in S \cap V$ by $y$ if $(f_A(x),y) \in L_V$
\State $L \gets L \cup \{x \in S \cap V \mid f_A(x) \in L_V\}$, $S \gets S\setminus L$
\EndWhile

		\end{algorithmic}
	\end{algorithm}

\begin{proof}(Proof of \Cref{th halfspace})
    In the first step, we show the correctness of \Cref{alg halfspace}. In each round of \Cref{alg halfspace}, we find a subspace $V$ that contains $k/d$ fraction of the unlabeled data in $S$ and a matrix $A$ that can make $f_A(S \cap V)$ in approximate radially isotropic position. Denote by $B$ the unit sphere in $A(V)$, we notice that for every $x \in V$, we have 
\begin{align*}
    \sgn(w^* \cdot x) =  \sgn(A^{-T}w^*  \cdot A x) = \sgn(A^{-T}w^*  \cdot f_A(x)) = \sgn( \proj_{A(V)}(A^{-T}w^*)  \cdot f_A(x)),  
\end{align*}
which implies that we can view $f_A(V)$ to be labeled by a halfspace $v^*=\proj_{A(V)}(A^{-T}w^*)$ furthermore, $x$ and $f_A(x)$ have the same label. According to \Cref{th perceptron}, we know that each labeled example in the output of \Cref{alg perceptron} has the correct label with respect to $v^*$ and thus the corresponding original examples in $S \cap V$ are also labeled correctly. 

However, up to now, we have not shown the correctness of the algorithm. This is because when we call \Cref{alg perceptron} as a subroutine in \Cref{alg halfspace}, we are not able to make queries in the transformed subspace $A(V)$, since it could be the case that $A(V) \cap Y = \emptyset$. Instead, we have to simulate a query $q(Z,y)$ used by \Cref{alg perceptron} with a query $q(\{x \in V \mid f_A(x) \in Z\},y)$ in the original space.

To show such a simulation is successful, it suffices to show the simulation has the following two properties. First, every subset $Z \subseteq B$ contains some transformed example $f_A(x) \in f_A(S \cap V)$ if and only if $\{x \in V \mid f_A(x) \in Z\}$ contains some example $x \in S$. This property ensures that the transformed labeling domain $f_A(Y \cap V)$ also contains the transformed dataset $f_A(S \cap V)$.
Second, for every $f_A(x) \in Z \subseteq B$ if it is labeled $y$ by $v^*$ then $x$ is labeled $y$ by $w^*$. This implies that the $q(Z,v^*,y)=q(\{x \in V \mid f_A(x) \in Z\},w^*,y)$ is always true. Thus, we can safely run \Cref{alg perceptron} with the simulated query. Since \Cref{alg halfspace} terminates when $(1-\alpha)$ fraction of the examples have been labeled and every labeled example has the correct label, we finish showing the correctness of the algorithm.

In the second step, we bound the query complexity and the running time of the algorithm. We first upper bound the number of calls for \Cref{alg perceptron}. Since the transformed subspace $A(V)$ has dimension $k$, by \cite{Ver18}, we know that with probability at least some constant $c$, a random selected $w_0$ satisfies $\abs{w_0 \cdot v^*} \ge 1/2\sqrt{k}$. When this happens, according to \Cref{th perceptron}, we know that \Cref{alg perceptron} will correctly label $1/4k$ fraction of the transformed dataset $f_A(S \cap V)$ and \Cref{alg halfspace} will enter next round. Thus, in expectation \Cref{alg perceptron} will be called constant times and each call will make $O(k^2\log^2 k) \le O(d^2\log^2 d)$ queries. According to \Cref{th forster}, we know that in each round $\card{S \cap V}/\card{S} \ge k/d$ and $1/4k$ fraction of $S \cap V$ is labeled correctly. This implies that after $O(d\log (1/\alpha))$ rounds only $\alpha n$ examples are not labeled and \Cref{alg halfspace} will terminate. This implies the total query complexity is $O(d^3\log^2d \log (1/\alpha))$. In particular, by setting $\alpha=o(1/n)$, we know that by making $O(d^3\log^2d \log n)$ queries, \Cref{alg halfspace} perfectly label $S$. To upper bound the running time of the algorithm, we notice that in each round, we run \Cref{alg perceptron} and compute an approximate Forser's transform, each of time can be done in polynomial time. Since the total number of rounds is at most $O(d\log n)$, we know \Cref{alg halfspace} is also a polynomial time algorithm.

Finally, we upper bound the VC dimension of the query family $Q$. Notice that each query we make can be summarized as follows $\{x \in V \mid f_A(x) \in Z\}$, where $Z=\{x'\} \cup \{x\in A(V) \mid Cx \le d\}$, where $C$ has at most $O(d)$ constraints. Thus, each query is the interaction of $O(d)$ degree two polynomial inequalities and a subspace (unions with a single point), which has a VC dimension of $\Tilde{O}(d^3)$.
\end{proof}

We notice that when we run \Cref{alg perceptron}, the region $T$ in a query $(T,z)$ is a set of $O(d)$ linear inequalities. Since in our learning model, each query is binary, we have to use such region queries to do a binary search in order to find some point $x$ such that $(x\cdot w_t)(x\cdot w^*) \le 1/\poly(d)$. Thus, when we run \Cref{alg halfspace}, each query uses a region that is the interaction of $O(d)$ degree two polynomial inequalities, and a subspace. This is why the VC dimension of $Q$ in \Cref{th halfspace} is $\Tilde{O}(d^3)$. If we are in a stronger learning model, where a counter-example $x \in T \cap L$ with label $-z$ is also returned when $q(T,z)=0$, then the binary search approach in \Cref{alg perceptron} is not necessary. In this setting, in \Cref{alg perceptron}, each region is defined by a single halfspace, and thus the VC dimension of the query class $Q$ we use for \Cref{alg halfspace} will be improved to $\Tilde{O}(d^2)$.

\section{Learning A Specific Hypothesis Class via A Specific Query Class}\label{sec dimension}
Although \Cref{th general up} shows that given a hypothesis class $\H$ with VC dimension $d$, we can construct a query class $Q$ with VC dimension $O(d)$ so that using $Q$, we can design a learning algorithm with query complexity, we have also seen from \Cref{sec efficient} that if a hypothesis class has a good structure, a query class with VC dimension $O(\log d)$ or even constant is sufficient to achieve a query complexity of $O(\log n)$. So an interesting question is given a hypothesis class $\H$ and a query class $Q$, what is the query complexity of learning $\H$ with $Q$?
Such a question has been extensively studied in many works in the literature of exact learning such as \cite{angluin1988queries,balcazar2001general,balcazar2002new,chase2020bounds}. Many different combinatorial characterizations have been developed. As a by-product of \Cref{th general up}, we can also define a new combinatorial dimension to characterize the query complexity of using a specific query class $Q$ to learn $\H$.

\begin{definition}[Partial Labeling and Extension]
Let $\X$ be a space of examples and $\H$ be a hypothesis class over $\X$. 
Let $S \subseteq \X$ be a set of $n$ examples. A partial labeling $f$ over $S$ is a labeling function $f: S' \subseteq S \to \{-1,1\}$, where $S' \subseteq S$. We say hypothesis $h \in \H$ is an extension of $f$ if for every $x \in S'$, $f(x)=h(x)$. In particular, we denote by $H_f=\{h \in \H \mid h(x)=f(x), \forall x \in S'\}$ the set of extensions of $f$ in $\H$. 
\end{definition}

\begin{definition}[Generalized Teaching Tree]
    Let $\X$ be a space of examples, $\H$ be a hypothesis class over $\X$ and $Q$ be a query family. Let $S\subseteq \X$ be a set of $n$ examples. A generalized teaching tree $T_f$ for $f$ is a binary tree that satisfies the following properties.
    \begin{itemize}
        \item Each node $v$ of $T_f$ is associated with a subset $H_v$ of hypothesis in $\H$. The root of $T_f$ is associated with $\H$.
        \item Each internal node $v$ of $T_f$ is also associated with a query $q_v \in Q$.
        \item Denote by $v_l$ and $v_r$ the left child of an internal node $v$. $H_{v_l}:=\{h \in H_v \mid q_v(h)=0\}$, $H_r:=\{h \in H_v \mid q_v(h)=1\}$.
        \item The subset of hypothesis $H_v$ associated with a leaf $v$ is either a subset of $H_f$ or a subset of $\H \setminus H_f$.
    \end{itemize}
    
\end{definition}

\begin{definition}[Query Dimension]
    Let $\X$ be a space of examples. Let $\H$ be a class of hypotheses over $\X$ and $Q$ be a family of queries. For any $n \in N^+$, we define $s(n)$ to be the query dimension of $(\H,Q)$ as follows. 
    \begin{align*}
        s(n)= \max_{S\subseteq \X, \card{S}=n} \max_{f:\text{partial labeling over }S} \min \{\text{depth}(T_f) \mid T_f: \text{a generalized teaching tree for }f\}.
    \end{align*}
\end{definition}

\begin{theorem}\label{th dimension}
    Let $\X$ be a space of examples. Let $\H$ be a class of hypotheses with VC dimension $d$ and $Q$ be a family of queries. Let $s(n)$ be the query dimension of $(\H,Q)$. 
    \begin{itemize}
        \item For any deterministic active learner $\A$, there is a subset $S \subseteq \X$ of $n$ such that if $\A$ makes less than $s(n)$ queries then there is some $h^*$ and some $x\in S$ such that $\A$ labels $x$ incorrectly.
        \item There is an active learner $\A$ such that, for every subset of $n$ example and every $h^* \in \H$, $\A$ makes $O(s(n)d\log n)$ queries from $Q$ and labels every example in $S$ correctly.
    \end{itemize}
\end{theorem}

\begin{proof}
    Clearly, given any set $S$ of $n$ examples, every active learning algorithm $\A$ constructs a  generalized teaching tree for every partial labeling function $f$, because each leaf of the tree corresponds to the hypothesis in $H$ that labels $S$ in the same way. In particular, let $S^*$ be the set of $n$ examples that achieves the maximum in the definition of $s(n)$, then the number of queries needed for $\A$ to label $S$ is at least the depth of the teaching tree it constructs which is larger than the number of queries to teach any partial labeling $f$.

    On the other hand, by \Cref{lm set construct}, we know that for every hypothesis class $H'$ and every set of $n$ examples $S$, there is some partial labeling $f$ such that $\card{(H'_f)_S} / \card{H'_S} \in [1/3,2/3]$. ($H'_S$ is $H'$ restricted at $S$ and $(H'_f)_S$ is $H'_f$ restricted at $S$.) Thus, with at most $s(n)$ queries, we are able to check if the target hypothesis is in $H'_f$ or not and shrink the hypothesis class by a factor of constant. Thus after making $O(s(n)d \log n)$ queries, we label $S$ correctly.
\end{proof}

\end{document}